\documentclass{article}

\PassOptionsToPackage{numbers, compress}{natbib}

\usepackage[preprint]{neurips_2021}

\usepackage[utf8]{inputenc} 
\usepackage[T1]{fontenc}    
\usepackage{hyperref}       
\usepackage{url}            
\usepackage{booktabs}       
\usepackage{amsfonts}       
\usepackage{nicefrac}       
\usepackage{microtype}      
\usepackage{xcolor}         

\usepackage{microtype}      
\usepackage{subcaption}
\usepackage{float}
\usepackage{blindtext}
\usepackage{booktabs,paralist}
\usepackage{amsfonts,amsmath,amssymb,amsthm} 
\usepackage{nicefrac,microtype,physics} 
\usepackage{color}
\usepackage{graphicx,subcaption}
\usepackage{thmtools,thm-restate}
\usepackage{enumitem}
\usepackage{adjustbox}
\usepackage{authblk}
\usepackage{hyperref}
\usepackage{url}

\newcommand{\remove}[1]{{}}

\renewcommand{\c}{\mathcal}

\renewcommand{\hat}{\widehat}
\renewcommand{\tilde}{\widetilde}

\renewcommand{\u}[1][]{\ifthenelse{\equal{#1}{}}{\mathbf{w}}{\mathbf{w}^{(#1)}}}
\newcommand{\w}[1][]{\ifthenelse{\equal{#1}{}}{\boldsymbol{\beta}}{\boldsymbol{\beta}^{(#1)}}}
\newcommand{\z}[1][]{\ifthenelse{\equal{#1}{}}{\mathbf{z}}{\mathbf{z}^{(#1)}}}
\newcommand{\h}[1][]{\ifthenelse{\equal{#1}{}}{{\mathbf{h}}}{{\mathbf{h}^{(#1)}}}}
\newcommand{\bdelta}[1][]{\ifthenelse{\equal{#1}{}}{{\boldsymbol{\delta}}}{\boldsymbol{\delta}^{(#1)}}}
\newcommand{\brho}[1][]{\ifthenelse{\equal{#1}{}}{{\boldsymbol{\rho}}}{\boldsymbol{\rho}^{(#1)}}}

\newtheorem*{assumption*}{Assumptions}
\newtheorem*{definition*}{Definition}
\newtheorem{theorem}{Theorem}
\newtheorem{lemma}[theorem]{Lemma}

\usepackage{comment}

\usepackage{newfloat}
\DeclareFloatingEnvironment[name={Supplementary Figure},fileext=lof]{suppfigure}
\DeclareFloatingEnvironment[name={Supplementary Table},fileext=lof]{supptable}

\title{Local Convolutions Cause an Implicit Bias towards High Frequency Adversarial Examples}

\author[1,2,+]{Josue Ortega Caro}
\author[3]{Yilong Ju}
\author[1]{Ryan Pyle}
\author[4]{Sourav Dey}
\author[5]{Wieland Brendel}
\author[1]{Fabio Anselmi}
\author[1,6]{Ankit B Patel}
\affil[1]{Department of Neuroscience\\
   Baylor College of Medicine\\
   Houston, TX, 77030}
\affil[2]{Quantitative and Computational Bioscience Program \\
   Baylor College of Medicine \\
   Houston, TX, 70030}
\affil[3]{Department of Computer Science \\
   Rice University \\
   Houston, TX, 77005}
\affil[4]{ Manifold AI}
\affil[5]{University of Tübingen \\
   Germany \\}
\affil[6]{Department of Electrical and Computer Engineering \\
   Rice University \\
   Houston, TX, 70005}
\affil[+]{Correspondence to: josue.ortegacaro@yale.edu}

\begin{document}

\maketitle

\begin{abstract}
Adversarial Attacks are still a significant challenge for neural networks. Recent efforts has shown that adversarial perturbations typically contain high-frequency features, but the root cause of this phenomenon remains unknown. Inspired by theoretical work on linear full-width convolutional models, we hypothesize that convolutional operations with \emph{localized kernels} are \emph{implicitly biased to learn high frequency features}, and that this is one of the root causes of \emph{high frequency adversarial examples}. To test this hypothesis, we analyzed the impact of different choices of linear and \textit{nonlinear} architectures on the implicit bias of the learned features and the adversarial perturbations, in spatial and frequency domains. We find that, independently of the training dataset, convolutional operations have higher frequency adversarial attacks compared to other architectural parametrizations, and this phenomenon is exacerbated with stronger locality of the convolutional kernels. The explanation for the kernel size dependence involves the Fourier Uncertainty Principle: a spatially-limited (local in the space domain) filter cannot also be frequency-limited (local in the frequency domain).
Using larger convolution kernel sizes or avoiding convolutions (e.g. by using Vision Transformers or MLP-style architectures) significantly reduces this high frequency bias. 

Looking forward, our work strongly suggests that understanding and controlling the implicit bias of architectures will be essential for achieving adversarial robustness. 
\end{abstract}

\section{Introduction}\label{sec:intro}
Despite the enormous progress in training neural networks to solve hard tasks, they remain surprisingly and stubbornly sensitive to imperceptibly small worst-case perturbations known as \textit{adversarial examples}. There has been a significant amount of work regarding the nature and structure of adversarial examples \cite{gilmer2018adversarial, mahloujifar2019curse,tanay2016boundary,ford2019adversarial,fawzi2018adversarial,bubeck2018adversarial,goodfellow2014explaining,schmidt2018adversarially,ilyas2019adversarial}. One particular experimental observation is that, oftentimes,  adversarial examples show a large  amount of energy content in the high frequencies \cite{yin2019fourier}. Furthermore, previous work has shown that adversarial examples are not just random perturbations of the input space but they contain dataset-specific information which is informative of class decision boundaries \cite{ilyas2019adversarial}. A natural question is then: \emph{does the concentration of energy in the high frequencies reflect some features of the task?} Wang et al., 2020 showed that high frequency features are necessary for high generalization performance of different models trained on CIFAR10 \cite{wang2020high}. They argue that learning these high frequency features is effectively a data-wise phenomenon, showing that models that use lower-frequency features had lower accuracy on CIFAR10. These results point out that the presence of high frequency features are conditioned to their usefulness in this dataset. Moreover, Maiya et al., 2021 provided evidence that different datasets produced adversarial examples with different concentration of energy in frequency domain which are related to the dataset statistics \cite{maiya2021frequency}. Furthermore, previous work has also shown that sensitivity to certain frequency based features can be altered by decreasing the reliability of those features in the dataset via data augmentations~\cite{li2022robust,hermann2020origins,geirhos2018imagenet}.  All together these results suggest that feature selection, and specifically high frequency features, are mainly driven by dataset statistics.

However, even though different datasets have different features that are correlated with the target function, there are usually more useful features than necessary for equal performance \cite{ilyas2019adversarial}. So why is any given model learning frequency based features and more specifically high frequency features? Different theories have been produced for understanding robustness and generalization of neural networks through a frequency lens. For example, Universal Adversarial Perturbations (UAPs) is a method to compute which directions in input space are neural networks sensitive to~\cite{tsuzuku2019structural}. The authors of this work found that convolutional neural networks were sensitive to noise in the Fourier Basis, where other models such as MLPs are not so. Moreover the work about Neural Anisotropic Directions (NADs) found that neural networks use linearly separable directions for classification and these directions can be approximated by directions in the frequency domain \cite{ortiz2020hold,ortiz2020neural}.\\

Here we take a different and more principled approach by relying on the concept of \textit{implicit bias} to answer this question. The idea behind implicit bias is that the loss landscape of an overparameterized network has infinitely many global minima, and which global minimum one converges to after training depends on many factors, including the choice of model architecture and parametrization \cite{gunasekar2018implicit, yun2020unifying}, the initialization scheme \cite{sahs2020a} and the optimization algorithm \cite{williams2019gradient, woodworth2020kernel, sahs2020shallow}. The implicit bias of state-of-the-art models has been shown to play a critical role in the generalization of deep neural networks \cite{li2019enhanced,arora2019exact}. Recent theoretical work~\cite{gunasekar2018implicit} on $L$-layer \emph{deep linear networks} proved that (i) fully connected layers induce a depth-independent ridge ($\ell_2$) regularizer in the spatial domain of the networks weights whereas, surprisingly, \emph{full-width convolutional layers} induce a depth-dependent \textit{sparsity} ($\ell_{2/L}$) regularizer in the \textit{frequency} domain. Linear full-width convolutional models are different from the high-performance convolutional neural networks (CNNs) used in practice. Nonetheless, we hypothesised that similar mechanisms might play a role for \emph{deep nonlinear models with local convolutions}: the high frequency nature of commonly-found adversarial perturbations is due to the implicit bias induced by the specific architectural choice and not only by the dataset statistics. More formally, we propose the Implicit Fourier Regularization (IFR) hypothesis (see Figure~\ref{fig:IFRhypo}):

\begin{center}
\textit{Implicit regularization in frequency domain, directly caused by the translation invariance and spatially limited nature of bounded-width convolutions, yields a strong bias towards high frequency features and adversarial perturbations.}
\end{center}

\begin{figure*}[t]
\begin{subfigure}{\textwidth}
  \centering
  \includegraphics[width=1\textwidth]{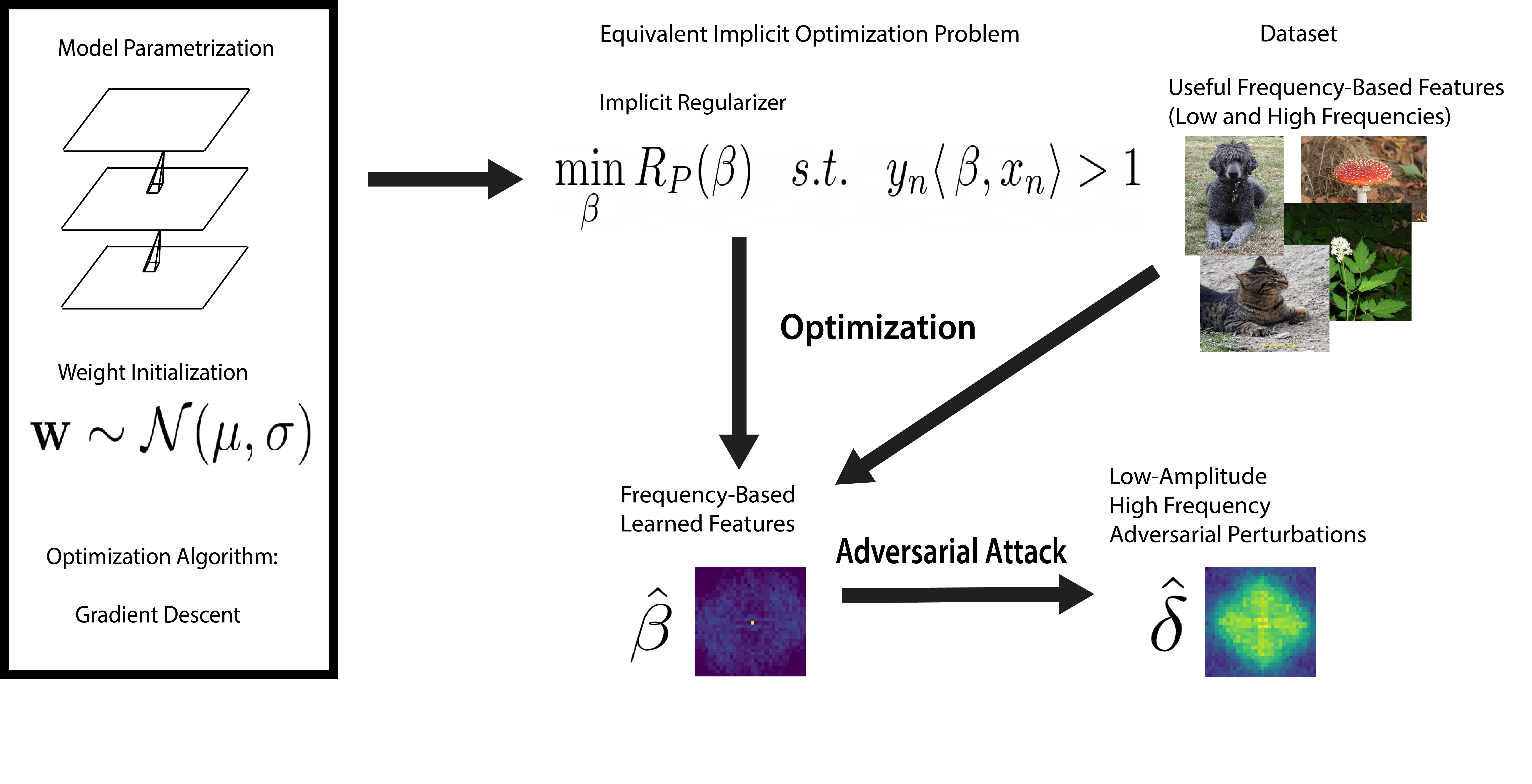}
\end{subfigure}
\caption{The Implicit Fourier Regularization (IFR) Hypothesis. Model parametrization and weight initialization induces an implicit optimization bias. This leads to different learned features $\beta$ and therefore, different adversarial perturbations $\delta$. In particular, implicit regularization in the frequency domain, directly caused by the translation invariance of the convolution operation, is necessary for the evolution of frequency-based features. Furthermore, the spatially-limited nature of common local convolutional operations encourages the learned features $\beta$ to possess energy in \textit{higher} frequencies.}\label{fig:IFRhypo}
\end{figure*}

The IFR hypothesis suggests that if a datasets presents useful high frequency features, models with bounded-width convolutional parametrizations will be biased to learn these features. This is because local convolutions induce a specific implicit bias, encouraging the linear features to possess more energy in \textit{high} frequencies which, in turn, induce high frequency adversarial perturbations. More broadly, our work establishes the relationship between the implicit regularization due to model parametrization and the structure of adversarial perturbations.

\section{Experimental Results}\label{sec:expt-results}

In the following we are going to test different aspects of the IFR principle. In particular in section \ref{sec:fullwidth} we analyze and compare the Fourier spectrum of the weights and adversarial perturbations of \textit{full width architectures} (convolutional or not, deep or not, linear or not, where the support of the weights is the full image).
In section \ref{sec:boundedwidth} we repeat the analysis for bounded width convolutional models and compare the spectral analysis with that of the full width models. Then, in section \ref{sec:translinv} and \ref{sec:nonconv} we analyze how convolutional weights sharing impacts the Fourier spectrum of the weights and adversarial perturbations. In section \ref{sec:shortcutfeats} we further test the model spectral bias injecting explicitly frequency-targeted shortcuts in the dataset and analyze the different models performances. Finally in section \ref{sec:imagenet} we test if the results obtained in the previous sections extends to a zoo complex models trained on Imagenet.

\subsection{Confirming Gunasekar et al, 2018~\cite{gunasekar2018implicit} results and extension to full width non linear models. Relationship to adversarial perturbations.}\label{sec:fullwidth}

To establish the relationship between implicit regularization and adversarial perturbations, we based our experiments on the recent theoretical work of Gunasekar et al., 2018 where the authors have shown that shallow linear convolutional neural networks ($L$ \textit{full-width circular} convolutional linear layers followed by one \textit{fully connected} linear layer) induce an implicit sparsity-promoting regularizer in the Fourier domain. Specifically, the regularizer is: $\c{R}_{FWC}(\beta_{L}) = \| \hat{\beta}_{L} \|_{\frac{2}{L}}$ where $\hat \beta_{L} := \mathcal{F} \beta_{L}$ is the Discrete Fourier Transform (DFT) of the end-to-end linear transformation $\beta_{L} := \prod_l^L W_l$ represented by the Full Width Convolutional network (FWC) with $W_{l}$ weights and $L$ the number of layers \cite{gunasekar2018implicit}. 
The authors proved that different linear layer parametrizations induce different implicit regularizers in the objective function. Thus changing the parametrization of the linear layer of a model induces different learned features $\beta$. Here, we argue that \textit{a similar effect should also be present in the structure of the adversarial perturbations $\delta$}. To test our hypothesis, we selected the two model parametrizations considered in Gunasekar et al, 2018, namely fully connected and full-width convolutional models. In addition, we used shallow (one hidden layer) and deep (three hidden layers) versions of these models, along with linear and nonlinear activation functions. These models were trained on 5 different datasets CIFAR-10 dataset~\cite{krizhevsky2009learning} (MIT), CIFAR100~\cite{krizhevsky2009learning}, MNIST~\cite{MNIST}, FashionMNIST~\cite{xiao2017fashionmnist}, and SVHN~\cite{SVHN} using PyTorch~\cite{paszke2019pytorch} (performances in Supp. Table ~\ref{tbl:accuracy}). 

Our goal is to analyze the relationship between the features $\beta$ learned by these models and their corresponding adversarial perturbations $\delta$ (See Section~\ref{sec:methods} for methodological details).

\begin{figure*}[t]
  \centering  \includegraphics[width=0.8\textwidth]{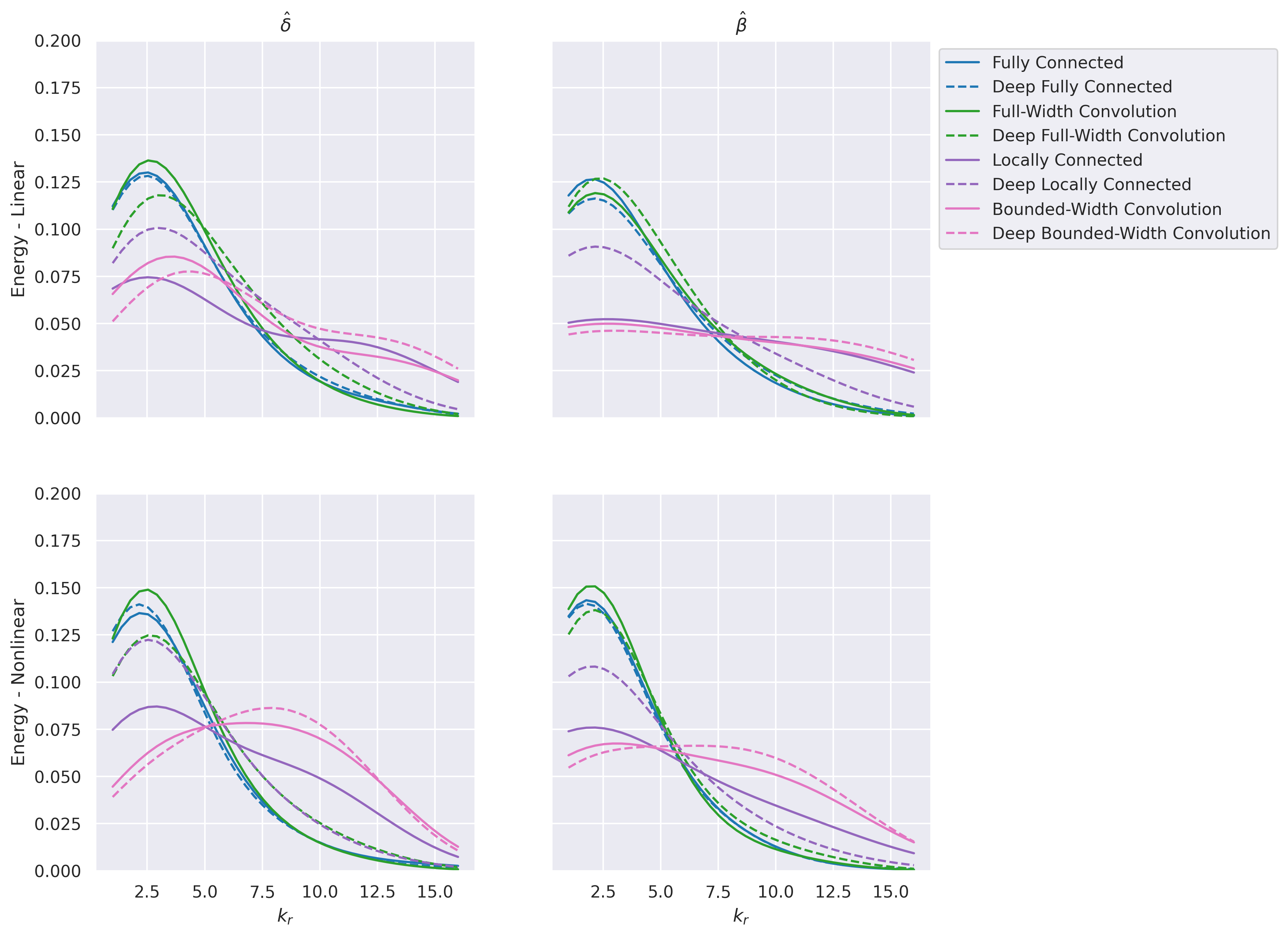}\caption{Radial Integral of the average frequency spectrum of adversarial perturbation ($\hat{\delta}$) and input-output weights ($\hat{\beta}$) for Fully Connected (blue), Full-Width Convolutional (green), Locally Connected (purple) and Bounded-Width Convolutional (pink) models for PGD-Linf Attack for CIFAR10. $k_r$ is the radial frequency. We can observe the correlation between the input-output weights and the adversarial perturbations for each model. This results also hold for different attacks (See Sup. Figures~\ref{fig:attacks_cifar10},~\ref{fig:attacks_cifar10_relu})}\label{fig:radial_cifar10} 
\end{figure*}

In order to do so we plotted, in Figure~\ref{fig:radial_cifar10}, for $\beta$ and $\delta$, the distribution of energy associated to the different frequencies in the spectrum of the perturbation averaging over all 2D directions (radial integral).
The Figure shows the radial integral of the adversarial perturbations ($\hat{\delta}$) for PGD-Linf attack \cite{madry2017towards} and that of the input-output weight ($\hat{\beta}$) in the case of CIFAR10  (for the other datasets see Supp Section \ref{sec:different_datasets}).  We observe that:

\textit{The full-width convolutional linear model has a sparser spectrum of $\beta$ compared to the fully connected model, and that when the model has more layers this difference increases (Deep FWC vs Deep FC). This is an experimental confirmation of the theoretical result in \cite{gunasekar2018implicit} for linear models}.

The results also show that $\beta$ and $\delta$ share the same spectral structure across all models we analyzed, indicating that the implicit regularizer induced by the model parametrization affects the structure of the adversarial examples. To further verify our results we quantified the different sparsity levels of the two models, measuring the $\ell_{1}$ and $\ell_{2}$ norm of the Fourier spectrum of $\beta$ for both the shallow and deep full-width and fully connected architectures as reported in Fig~\ref{fig:bar_norm}.

\begin{figure}[h]
    \centering
    \includegraphics[width=0.48\textwidth]{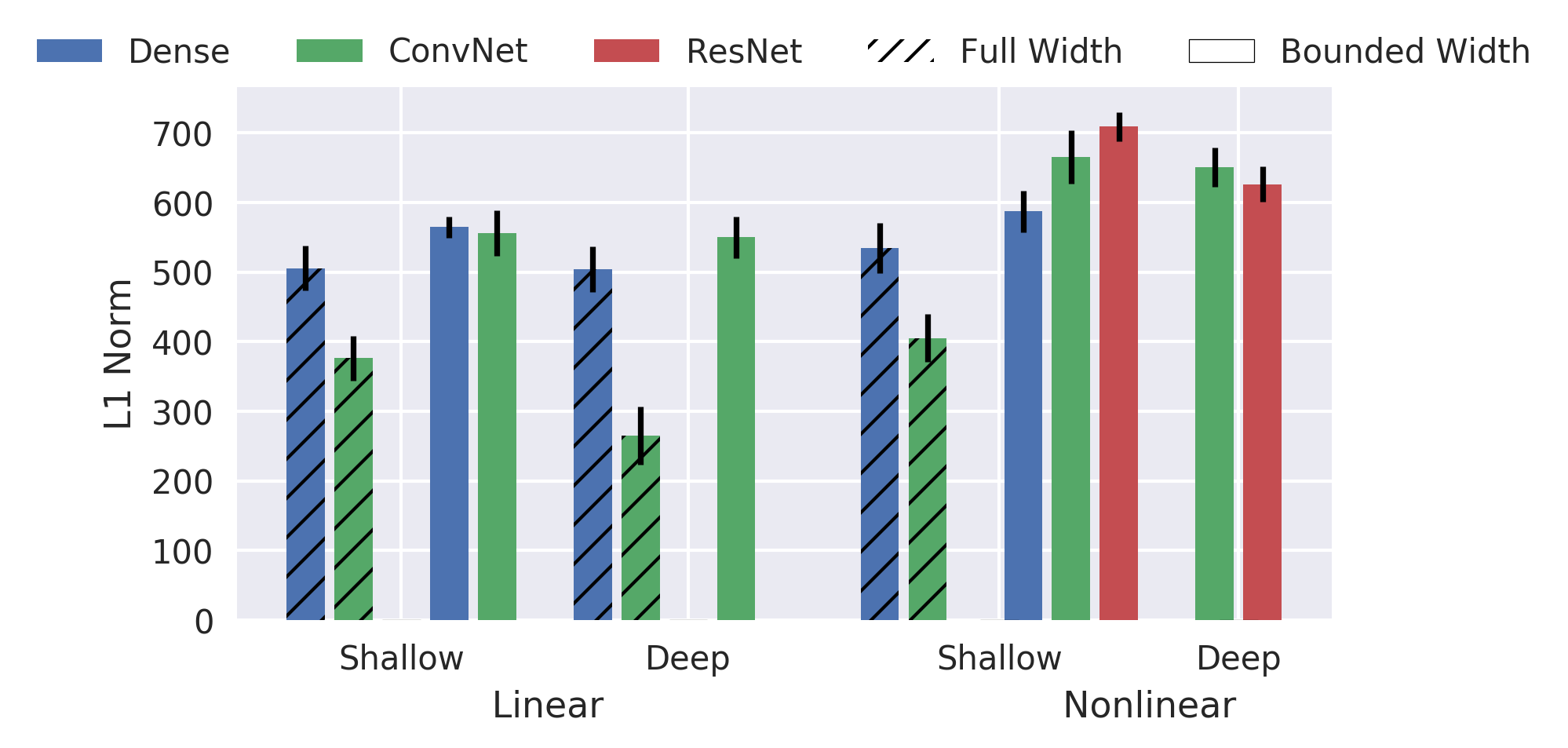}
    ~
    \includegraphics[width=0.48\textwidth]{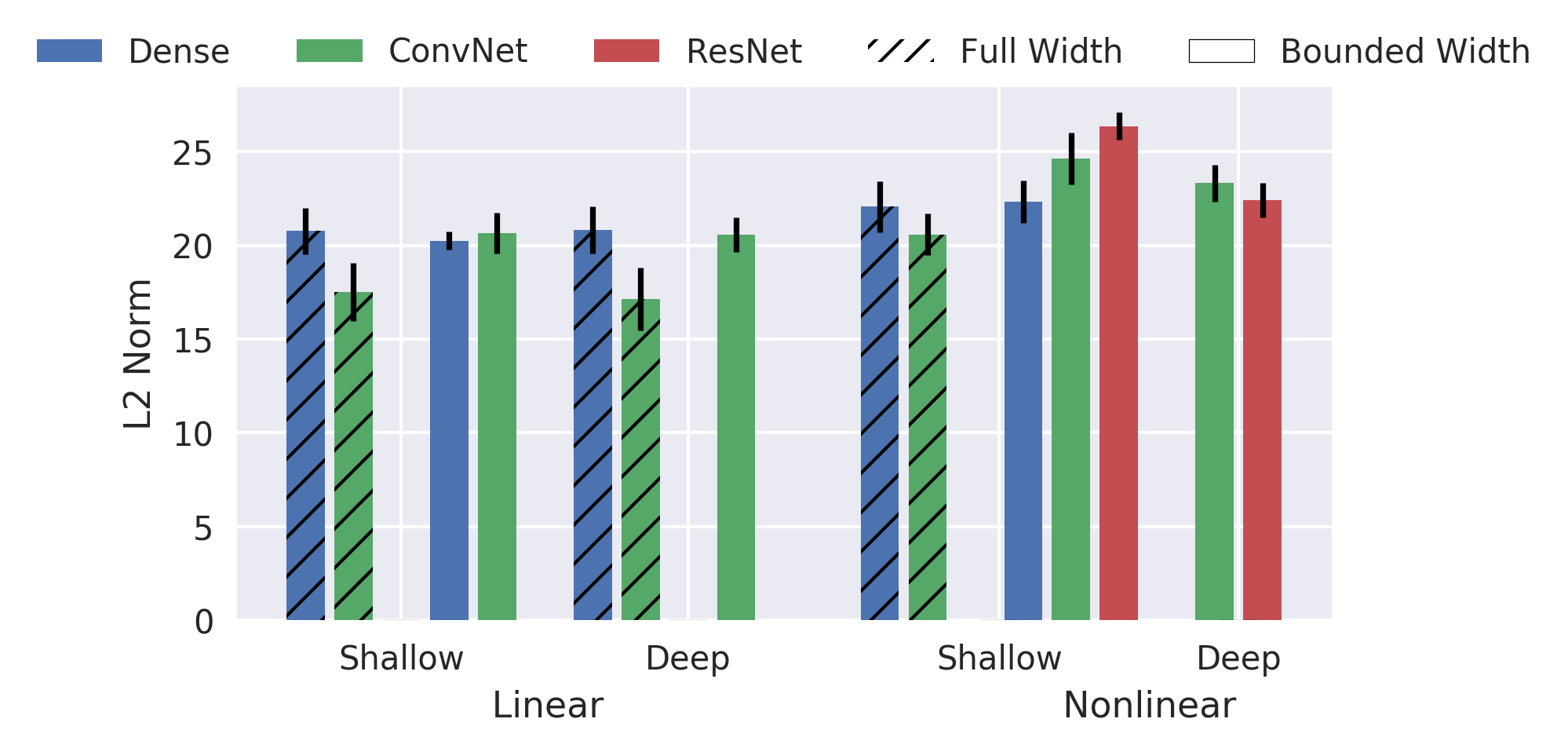}
    \caption{Comparison of Adversarial Perturbation Norms $\| \hat\delta \|_p$ for $p = 1,2$ for different linear and nonlinear models. The relations between norms of nonlinear models mirror and accentuate those of linear models. See Results for detailed comparisons.}\label{fig:bar_norm}
\end{figure}

Next, we tested if a similar implicit regularization holds for \textit{nonlinear models}. As shown in Figure~\ref{fig:bar_norm} the full-width convolutional model has a sparser spectrum for both $\beta$ and $\delta$ compared to the fully connected model, confirming our hypothesis.


\subsection{Bounded-Width Convolutional Models Analysis}\label{sec:boundedwidth}

One limitation of the analysis done in \cite{gunasekar2018implicit} and in the previous section is that full-width convolutional layers are not often used in common state-of-the-art architectures. Bounded-width convolutional models are instead usually employed, which have smaller kernel sizes than the width of the input. A natural question is therefore: what is the effect of the bounded-width convolutional model on $\hat{\beta}$ and $\hat{\delta}$?

To answer the question, we trained bounded-width convolutional models and generated both $\hat{\beta}$ and $\hat{\delta}$ with the same procedure as the previous section. In Figure~\ref{fig:radial_cifar10}, we can see that:

\textit{The bounded-width convolutional models (both linear and nonlinear) contain high frequency components in $\hat{\delta}$ and their presence grows with depth in the Deep Bounded-Width model (DBWC). Moreover, we observe that the $\beta$ corresponding to the bounded-width convolutional model is as sparse or less sparse (more dense) compared to the full-width counterpart as reported in Figure~\ref{fig:bar_norm}.}

Given that the theory in \cite{gunasekar2018implicit} only specifies a bias towards $\ell_\frac{2}{L}$-sparsity in the frequency domain (i.e. both in the low and high frequencies),  \textit{why do we observe a bias in the higher frequencies for the bounded-width convolutional model but not the full-width convolutional model?} Our experiments suggest that high frequency concentration is due to the fact that convolutional kernels are localized. We propose a theoretical explanation via the Fourier Uncertainty Principle -- i.e. a space-limited kernel \textit{cannot} be band-limited in frequency domain -- as the origin of the frequency bias. This reasoning can be made rigorous for a bounded-width linear convolutional model by a simple and straightforward extension of the results of Gunasekar et al, 2018 along with the invocation of energy concentration Uncertainty Principle, resulting in the following theorem:

\begin{restatable}{theorem}{thmucp}
Let $\beta := \star_{l=1}^{L-1} w_l$ be the end to end linear predictor, where $w_l$ is the weights for layer $l$ and $\star$ indicates convolution. Then decreasing the kernel size of each convolutional filter $w_{l} \in \mathbb{R}^{D}$ results in an increased concentration of energy in high frequencies for $\hat{\beta}$.  
\end{restatable}

A rigorous mathematical treatment can be found in the \ref{theory}.
The key intuition is that, due to the Fourier Uncertainty Principle, decreasing the support of the convolutional filter $w_{l}$ at layer $l$ causes an increase of its high frequency energy content. 

We directly tested our theory on the same models proposed by Gunasekar 2018 but with varying kernel sizes ($3\times3$, $7\times7$, $11 \times 11$, $15\times 15$, $21\times 21$, and $32\times 32$), with 1 or 3 hidden layers, and train them on Grayscale CIFAR10. This choice was made to keep the number of hidden channels equal to the number of input channels. Then we computed the frequency spectrum of $w_{l}$ and the cumulative weights up to layer $l$, $\beta_{l}$. Next we computed $\kappa_{high}$, defined as the fraction of energy outside of an interval $[-\frac{k}{2},+\frac{k}{2}]$, where $k$ is an specific frequency, divided by the total energy for each kernel. In Figure~\ref{fig:combine_kernel}, we observe that the fraction of energy is a function of the kernel size, where for smaller $k$, $\kappa_{high}$ is higher. Furthermore, as the model goes deeper, the difference is further exacerbated. We observe the same phenomenon both for deeper models and during training (Left). The complete plot with all intermediate layers can be found in the Supplementary Material.

\begin{figure}[h]
    \centering
    \includegraphics[width=\textwidth]{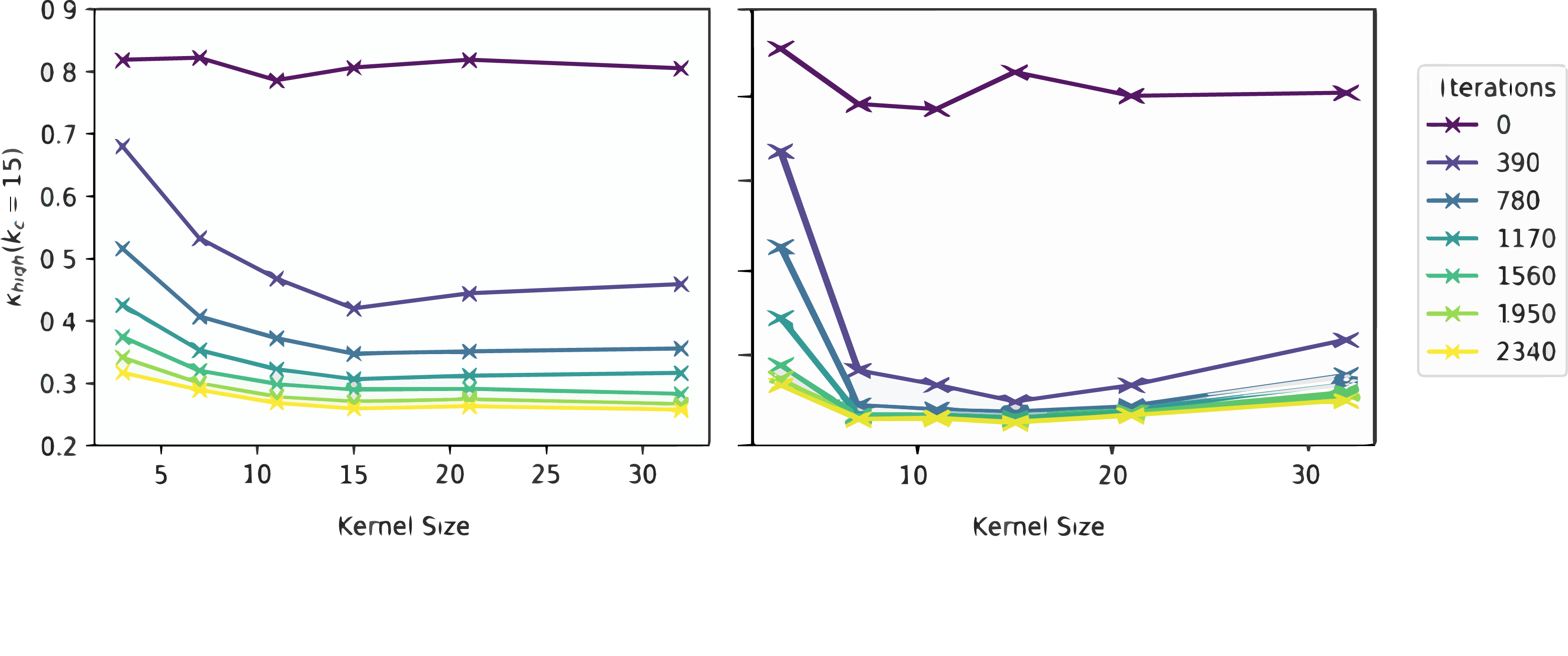}
    \caption{Concentration of energy $\kappa_{high}$ for $\kappa_{r}$ = 15 (the energy cut-off to consider high energy frequencies) for the input-output weights $\beta$ versus kernel size (3,7,9,11,15, 21, or 32). Results are obtained averaging 5 models. All models were trained for 40 epochs on Grayscale CIFAR10. Left: One hidden layer models and Right) Three hidden layer models.}
    \label{fig:combine_kernel}
\end{figure}

\subsection{Role of Translation Invariance in the frequency spectrum of $\beta$ and $\delta$}\label{sec:translinv}

Next, to further understand the origin of the high frequency bias, we investigated if the high frequency bias also present in a model with local kernels but not translation invariance (i.e. no shared weights). In other words we tested if convolutionality was essential in determining the high frequency bias. We trained a locally connected model (no convolutional weight sharing) with same kernel size on CIFAR10. In Figure~\ref{fig:radial_cifar10} we observe that the locally connected models do not have as much energy in the higher frequencies, particularly for the deep models. The results show that local connectivity alone is \textit{not} sufficient to bias the model towards learning high frequency features; translation invariance is also required. \textit{Thus, local convolutions are necessary to get a bias towards high frequency in the learned features and adversarial perturbations}.

\begin{figure*}[t]
\begin{subfigure}{\textwidth}
  \centering
  \includegraphics[width=\textwidth]{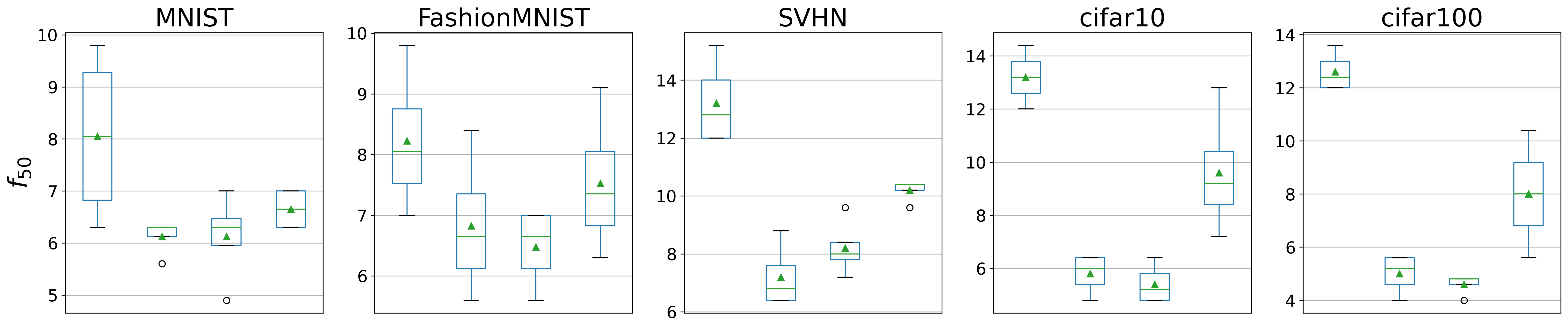}
\end{subfigure}
\hfill
\begin{subfigure}{\textwidth}
  \centering
  \includegraphics[width=1.015\textwidth]{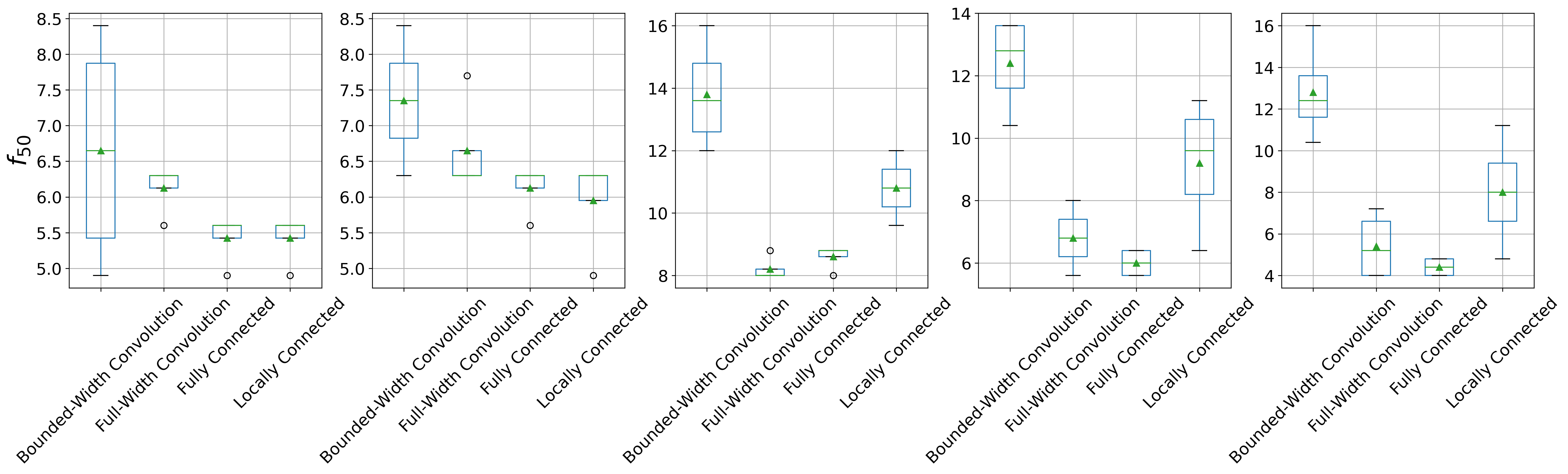}
\end{subfigure}
  \caption{Half Power Frequency ($f_{50}$) of (Top) adversarial perturbation ($\hat{\delta}$) and (Bottom) input-output weights ($\hat{\beta}$) for different datasets and models such as Fully Connected, Full-Width Convolutional, Locally Connected and Bounded-Width Convolutional models for PGD-Linf Attack  (See Sup. Figures~\ref{fig:attacks_cifar100_1D},~\ref{fig:attacks_FashionMNIST_1D},~\ref{fig:attacks_MNIST_1D},~\ref{fig:attacks_SVHN_1D} for 1D Radial Integral energy of each dataset)}\label{fig:f50_all_datasets} 
\end{figure*}

Clearly local convolutions are not the only factor.
Indeed, as Maiya et al \cite{maiya2021frequency} has shown, data-statistics is also an important element in determining the frequency content of adversarial perturbations. Thus we analyzed the compounding effect of model parametrization in the context of different datasets. Towards this goal we computed the cumulative 1D radial integral of $\hat{\beta}$ and $\hat{\delta}$ for MNIST, FashionMNIST, SVHN, CIFAR10 and CIFAR100. Then, we estimated the half power frequency ($f_{50}$) of each model (averaged across different depths and nonlinearities), i.e. the frequency at which we acumulate the $50\%$ total energy. In Figure~\ref{fig:f50_all_datasets}, we observe that the $f_{50}$ is larger for the Bounded-Width Convolutional Model compared to the Full-Width Convolution, Fully-Connected and Locally Connected models \textit{for the  all considered datasets} (See Sup. Figures~\ref{fig:attacks_cifar100_1D},~\ref{fig:attacks_FashionMNIST_1D},~\ref{fig:attacks_MNIST_1D},~\ref{fig:attacks_SVHN_1D} for Radial Integral energy of each dataset).  These results show that if high frequency useful features are present, the bounded-width convolutions will have more high frequency energy both in $\delta$ and $\beta$ compared to other model parametrizations. Furthermore, even in datasets with more low-frequency information content such as MNIST or FashionMNIST, there is still a smaller but significant bias towards higher frequencies in Bounded-Width Convolutional models.

\begin{figure*}[t]
\begin{subfigure}{\textwidth}
  \centering
  \includegraphics[width=\textwidth]{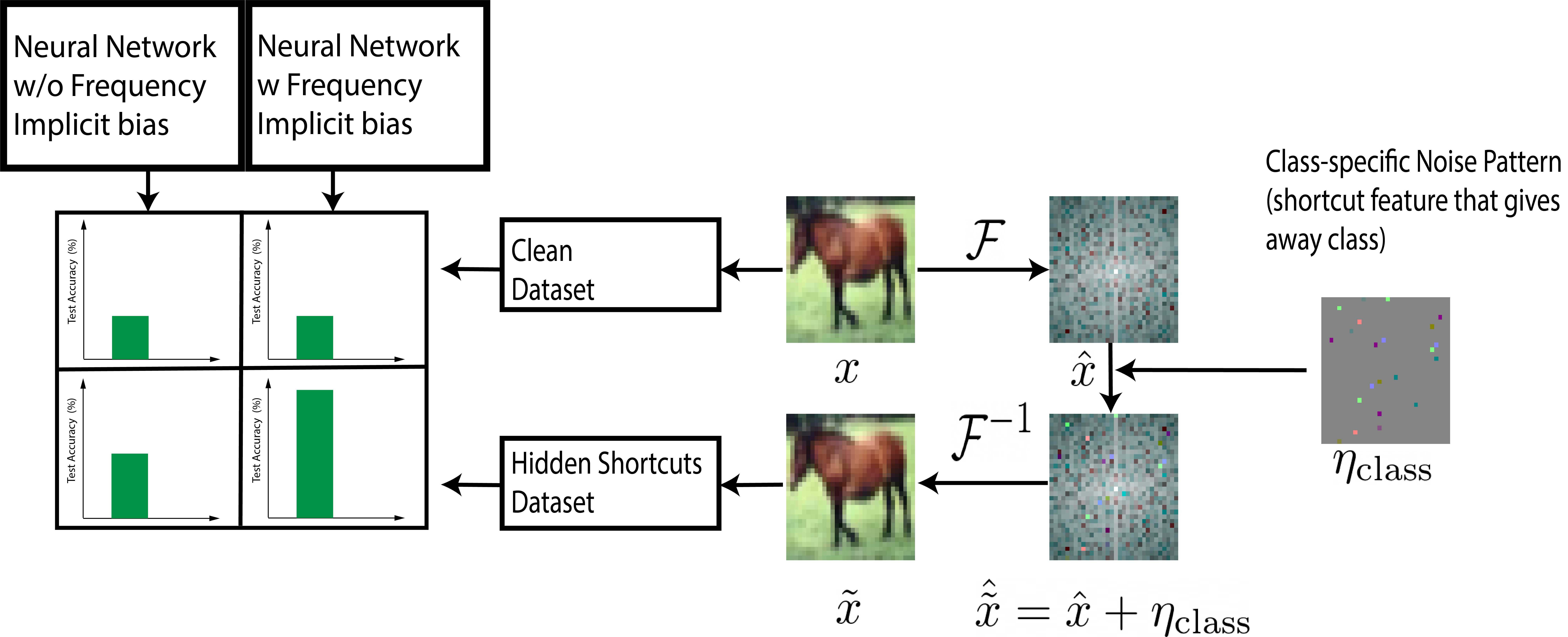}
\end{subfigure}
\caption{Steganography experiment. Class specific information $\eta_{class}$ introduced into Fourier Spectrum of training set image ($\hat{\tilde{x}}  = \hat{x} + \eta_{\textrm{class}} \: : \: \eta_{\textrm{class}} = M_{\textrm{class}} \odot (\epsilon * N_{\textrm{class}}), N_{\textrm{class}} \sim \mathcal{N}(0, 1), M_{\textrm{class}} \sim \textrm{Bernoulli}(\rho)$). Models with sparse frequency regularizer (Full-Width and Bounded-Width Convolutional models) should take more advantage of hidden shortcut features in the dataset, therefore leading to higher test accuracy.}\label{fig:steganography}
\end{figure*}

\subsection{Further testing different implicit bias via the Injection of Hidden Shortcut Features}\label{sec:shortcutfeats}

Given that the theory in Gunasekar et al., 2018 doesn't extend to the implicit regularization for nonlinear convolutional models, we decided to directly test if the implicit bias of nonlinear model parametrizations is similar to that of their linear counterparts.

To test our hypothesis, we took inspiration from the field of steganography which focuses on hiding shortcut features in images that are visually undetectable by humans~\cite{cheddad2010digital}. 
We introduced a class-correlated sparse shortcut features in Fourier space of every CIFAR-10 train and test set image (See Figure~\ref{fig:steganography}). Our hypothesis is that these class-dependent shortcut features will be mostly useful for model parametrizations that have an implicit regularizer that is sparse in the frequency domain. Therefore the model with the right implicit bias should achieve \textit{higher} test accuracy (for generation of the shortcut features see the Methods section, \ref{sec:methods}
). Furthermore, these experiments helped us to further analyze the relationship between the implicit bias due to model parametrization and specific dataset statistics because these new features will now be part of the dataset. Indeed if the model parametrization is not important all models will take advantage of those features and the dataset statistic is the main driver of feature selection.

\begin{table*}[h]
    \centering
    \caption{Performance on CIFAR-10 on base vs frequency-based dataset , mean $\pm$ std over 5 trials each. $\rho$ is the sparsity level of the added signal and $\epsilon$ is the scalar factor for the matrix. }
    \begin{adjustbox}{max width=\textwidth}
    \begin{tabular}{|c |c| c| c|}
        \hline
        &  Baseline & $\epsilon = .25, \rho = .2$ & $\epsilon = .25, \rho = .1$\\
        \hline
        Fully Connected Linear  & 40.8 $\pm$ .070 & 53.26 $\pm$ .075 & 45.27 $\pm$ 1.14\\
        Fully Connected Nonlinear  & 48.5 $\pm$ .084 & 59.6 $\pm$ .223 & 54.93 $\pm$ .497\\
        Full Width Convolutional Linear  & 41.8 $\pm$ .124 & 70.3 $\pm$ .805 & 79.6 $\pm$ 1.35\\
        Full Width Convolutional Nonlinear  & 52.4 $\pm$ 2.89 & 79.5 $\pm$ .912 & 87.2 $\pm$ 1.26\\
        Bounded Width Convolutional Linear  & 41.8 $\pm$ .046 & 97.28 $\pm$ .604 & 92.33 $\pm$ .679\\
        Bounded Width Convolutional Nonlinear  & 56.7 $\pm$ .489 & 98.8 $\pm$ .850 & 97.8 $\pm$ .475\\
        \hline
    \end{tabular}
    \end{adjustbox}
    \label{tab:poison}
\end{table*}

\begin{table*}[h!]
    \centering
    \caption{Performance on CIFAR-10 on base vs frequency-based dataset. Low Frequency $(\epsilon = 0.5)$, Medium Frequency $(\epsilon = 5e-02)$ and High Frequency $(\epsilon = 5e-04)$. Different $\epsilon$'s were chosen to match the signal to noise ratio of those frequencies bands and make the task as difficult as possible. Standard Deviation was calculated over 5 runs.}
    \begin{adjustbox}{max width=\textwidth}
    \begin{tabular}{|c| c| c| c|}
        \hline
        Models & Low Frequency & Medium Frequency & High Frequency\\
        \hline
        Full-Width Convolution (L=1)      & 100.0      & 41.03 & 41.38 \\
        Full-Width Convolution (L=3)  &  100.0      & 39.80 & 40.97 \\
        Bounded-Width Convolution (L=1)  &  100.0      &49.99 & 53.60 \\
        Bounded-Width Convolution (L=3)  & 100.0       & 94.17 & 98.47 \\
         Fully Connected (L=1)  & 99.78       & 41.45  & 41.38 \\
         Fully Connected (L=3)  & 100.0        & 46.37 & 41.21 \\
         Locally Connected (L=1)  & 100.0        & 44.39 & 46.23 \\
         Locally Connected (L=3)  & 100.0       & 47.36 & 54.73 \\
        \hline
    \end{tabular}
    \end{adjustbox}
    \label{tab:poison.freq}
\end{table*}

Table~\ref{tab:poison} shows the test accuracy of each linear and nonlinear model with the new dataset and different levels of sparsity in the frequency domain. \textit{We can observe that both the full-width and bounded-width models have higher performance than baseline with the new dataset}. However, the fully connected model, having a different bias, cannot take full advantage of the class-dependent shortcut features introduced into the dataset. This experiment reveals that the implicit regularizer for sparsity in the frequency domain \textit{is present in both linear and nonlinear convolutional models}.

In addition to these results, we further tested the high frequency bias of the bounded-width parametrization. By performing a new variant of the hidden features experiment, we now localize the information in the low, medium, or high frequencies by introducing the class-dependent signals characterized by frequency in specific bands of the spectrum of the training and testing set of CIFAR10 (See Section~\ref{sec:methods} for methodological details). 

Table~\ref{tab:poison.freq} shows the performance of the linear models with the new frequency-based class-dependent features. We observe that all models are able to use the low-frequency shortcut features in order to perform the task. Furthermore, when the cheat signal is introduced in the medium and high frequencies the full-width convolutional, fully connected, and locally connected models have a much more difficult time in selecting the signal. On the contrary the \textit{Bounded-Width Convolutional Model is able to perform the best in medium and high frequencies compared to the other models and the difference is maintained in deeper models}. This experiment not only confirms the presence of an implicit bias due to parametrization in nonlinear ReLU models, but also shows that when useful frequency-based information is present, models with localized convolutions are able to capture those features with more ease compared to other model parametrizations. We compare our results to Wang et al 2020, where they showed that the smoothness of the convolutional models final kernel is the main driver of low frequency bias. Here we show that  bounding the kernel has a stronger regularization on the frequency component of the features learned. Furthermore, Tsuzuku et al, 2019 showed that convolution-based models have a preference towards features in the Fourier basis. Here we expand on their results, asking what is the difference between bounded vs full-width convolutions and showing that not all convolution-based models are able to pick up every feature in the Fourier domain. 

\subsection{Non-convolutional state-of-the-art models do not exhibit a high frequency bias}\label{sec:nonconv}

Since local convolutions cause a bias towards high frequency features and adversarial perturbations, it is natural to ask: are high-performance deep models without convolutions less biased towards high-frequency features?  One possible architecture is the Vision Transformer (ViT), which has performed on par with convolution-based architectures in many tasks including object recognition \cite{dosovitskiy2020image}. Furthermore, recent work has shown that ViTs can be more robust to high frequency adversarial perturbations than ResNets \cite{shao2021adversarial}. 

\begin{figure*}[t]
\begin{subfigure}{0.49\textwidth}
  \centering
  \includegraphics[trim={0 0 0 0},clip,width=\textwidth]{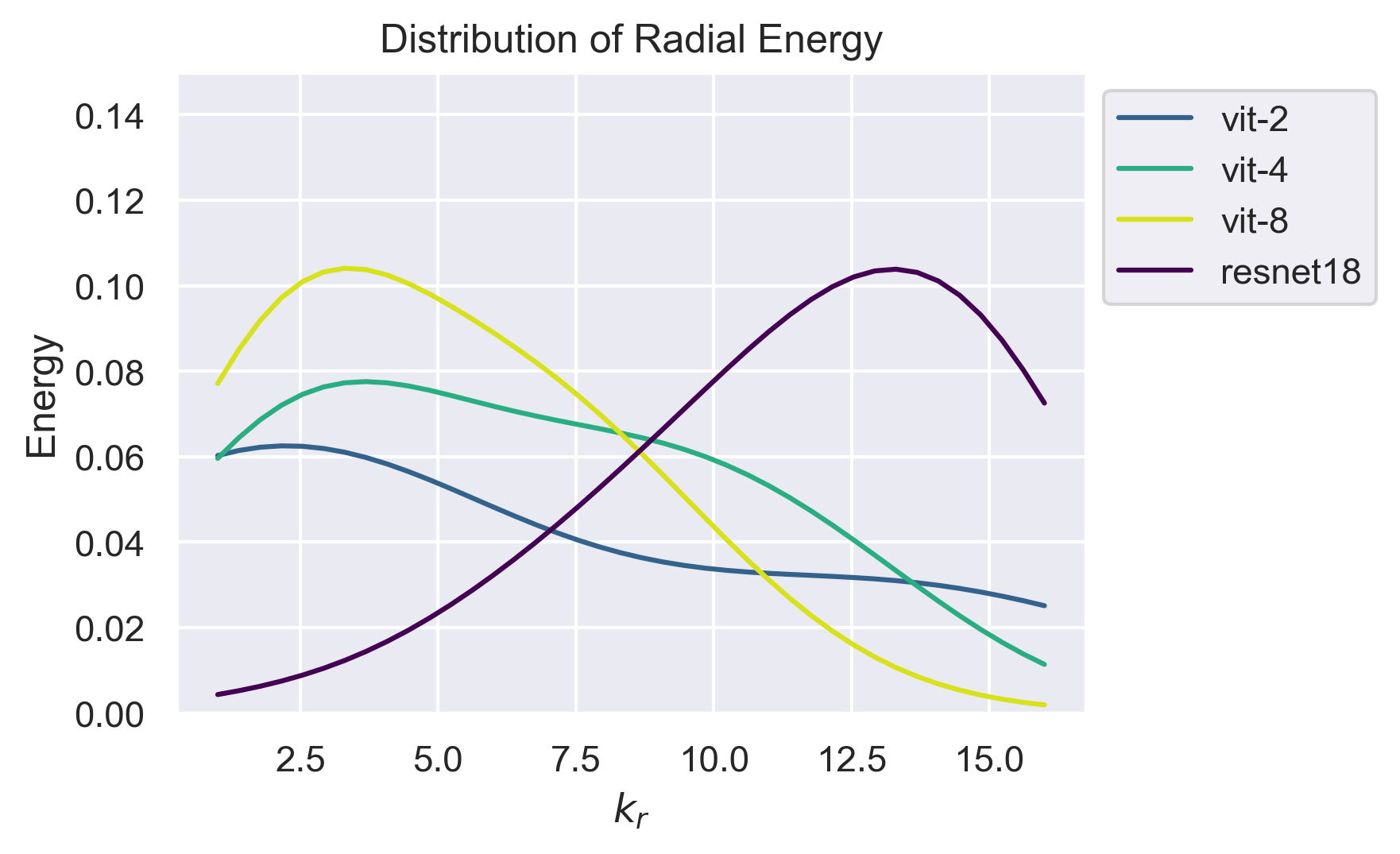}
  \caption{}\label{fig:vit}
\end{subfigure}
\hfill
\begin{subfigure}{0.49\textwidth}
  \centering
  \includegraphics[trim={0 0 0 0},clip,width=\textwidth]{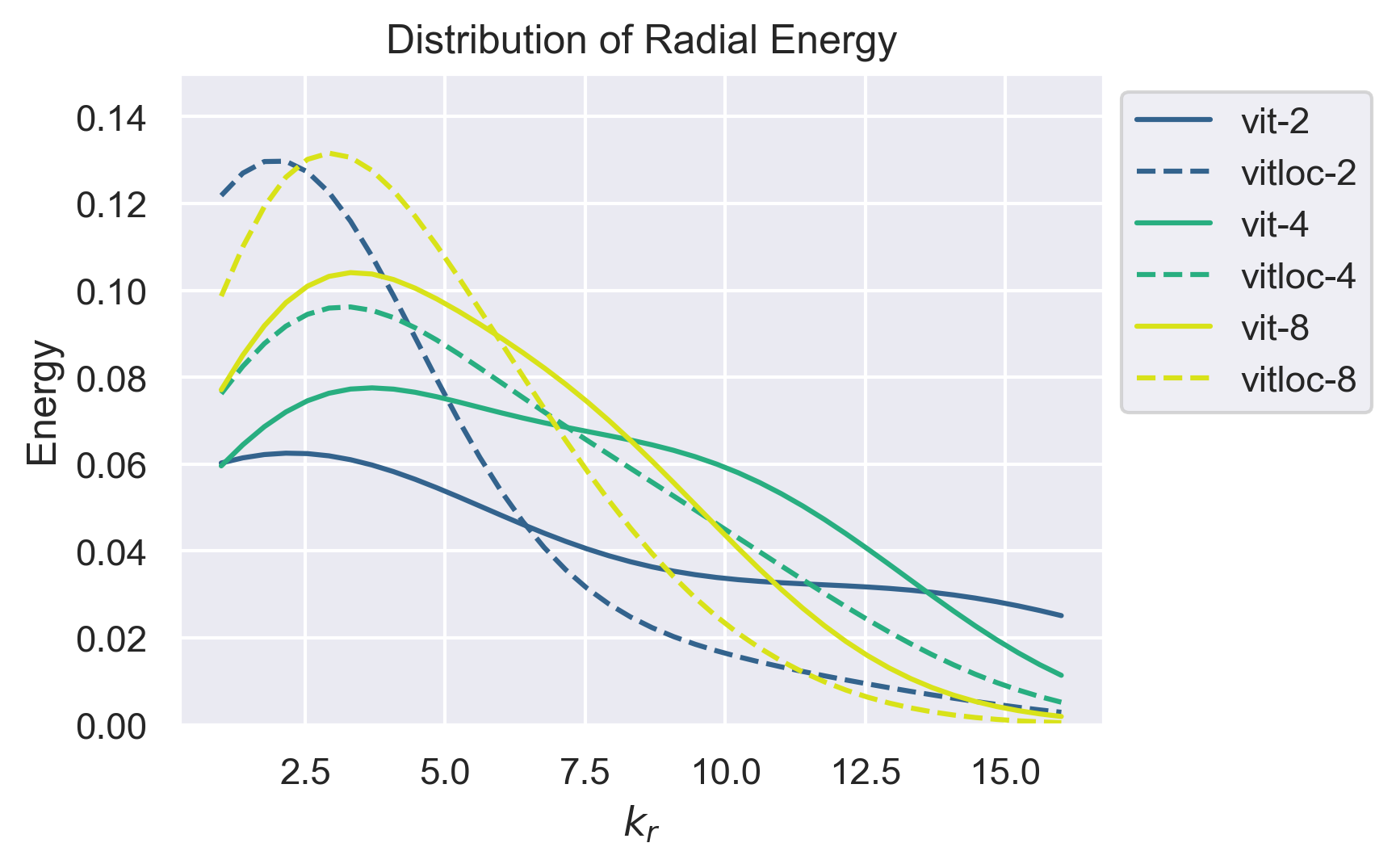}
  \caption{}\label{fig:vitloc}
\end{subfigure}
    \caption{Radial Energy of the Adversarial Perturbations of Vision Transformer and Local Vision Transformer models trained on CIFAR10. (a) All Vision Transformer models (vit-2, vit-4, vit-8) have more energy in lower frequencies than ResNet18. (b) Locally Connected model's (vitloc-2, vitloc-4, vitloc-8) spectrum have more energy in lower frequencies compared to their standard ViT counterpart.}\label{fig:transformer_cifar10} 
\end{figure*}

\begin{figure*}[h]
\begin{subfigure}{\textwidth}
  \centering
  \includegraphics[width=\textwidth]{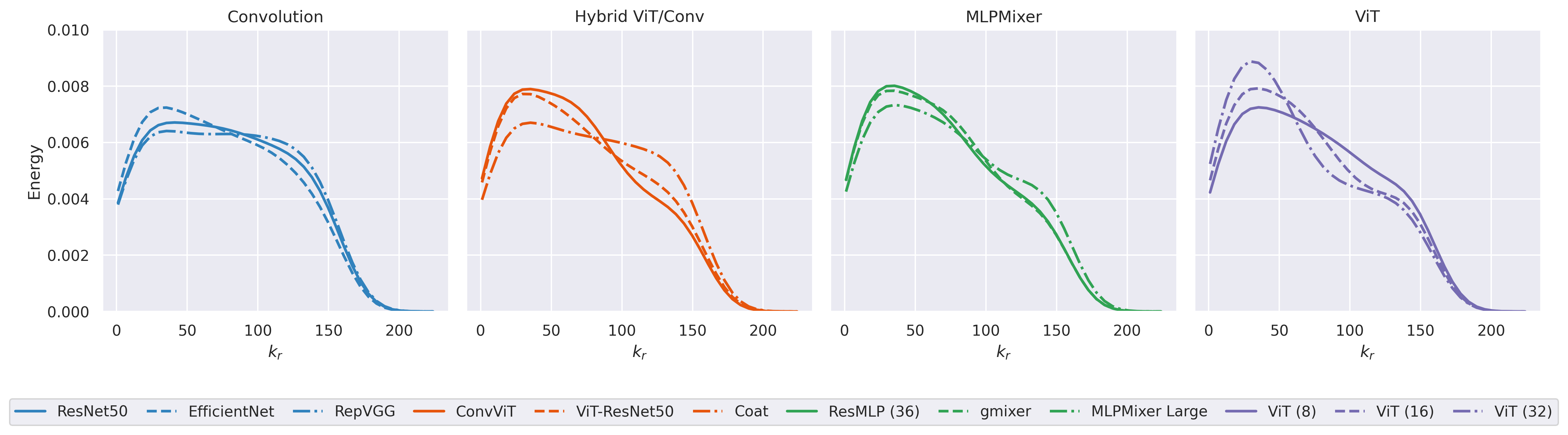}
\end{subfigure}
\begin{subfigure}{\textwidth}
  \centering
  \includegraphics[width=0.5\textwidth]{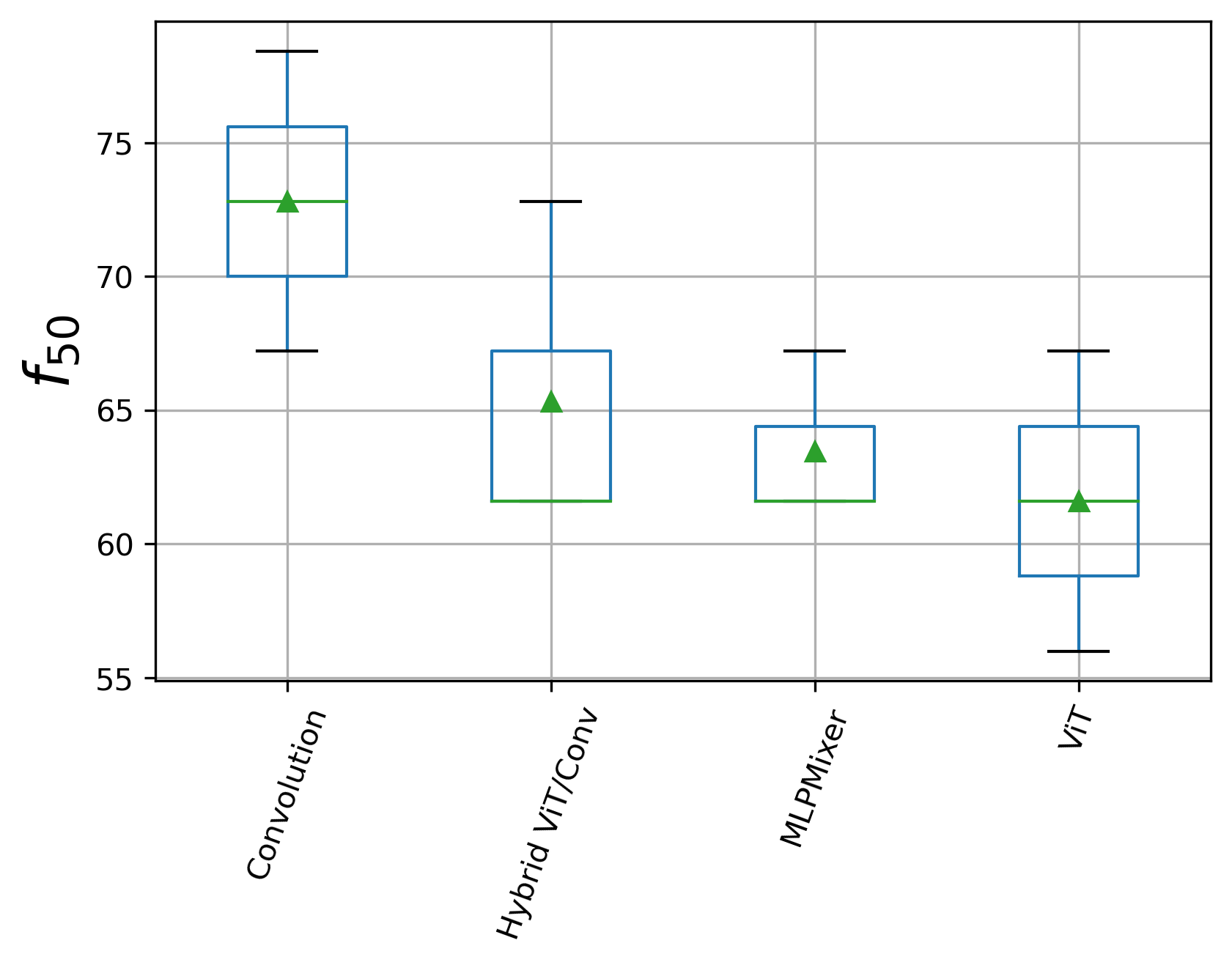}
\end{subfigure}
\caption{(Top) Radial Energy of the Adversarial Perturbations of Convolutional (Blue), Hybrid Vision Transformer and Convolutional Model (Orange), MLPMixer Variants (Green) and Vision Transformer (purple) models. (Bottom) Half Power Frequency $f_{50}$ for each model type. All the plots are computed over the 50000 validation set images. Convolutional Models have the most energy in the high frequencies across all models.}\label{fig:transformer_imagenet}
\end{figure*}

\textit{We hypothesize that the lack of convolutional operations in ViTs is responsible for the greater robustness to high frequency attacks}. In order to test our hypothesis, we trained a ViT with different patch sizes on CIFAR-10 (See Supp. Sec.~\ref{sec:supp-training-configs} for training details). All of these models achieved about 80$\%$ test accuracy. These are lower accuracies than state-of-the-art since those models are usually pretrained on ImageNet, a step we omitted being in the present work only interested in studying the effect of the model architecture. 
Then, we attacked the models with a PGD-Linf attack, the same used earlier in the paper. In Figure~\ref{fig:vit}, we observe the average frequency spectrum of ViTs with different word/patch sizes ($2 \times 2$, $4 \times 4$, $8 \times 8$), as compared to a ResNet18 with similar test accuracy (we stopped the training when the model had near 80\% test accuracy). We observe that all of the ViTs have more energy in the lower frequencies than the ResNet18. This strongly suggests that the robustness to high frequency attacks shown in \cite{shao2021adversarial} might be due to the lack of convolution layers. Furthermore, we observe that different patch/word sizes also have an impact on the spectrum, with smaller word sizes producing higher frequency spectra. This is consistent with our theoretical and empirical results comparing full-width and bounded-width convolutional models.

However, as others have suggested, some layers of the ViT can still be interpreted as convolutional layers, so the model is not entirely convolution-free~\cite{chen2021visformer}. In particular, the first embedding layer is shared across all image patches, and therefore is a convolutional layer with stride equal to the patch/kernel size. \textit{Can we remove this layer and shift the energy spectrum to even lower frequencies?} To test for this hypothesis we defined a new ViT architecture in which we removed the weight sharing in the first layer patch embedding by using a locally connected layer with the same kernel size and stride as the convolutional version, so that each patch has independent weights. We call this architecture the Vision Local Transformer (ViTLoc). Then, we trained this architecture on CIFAR-10 (see Performance in Sup. Table \ref{tbl:accuracy}) and computed the associated adversarial perturbations. In Figure~\ref{fig:vitloc}, we observe that every ViTLoc model has a lower frequency spectrum than its ViT counterpart. This is more evidence that avoiding the convolutional parametrization, even in very complex state-of-the-art models, can reduce the high frequency bias, and therefore, high frequency adversarial perturbations.

\subsection{Expanding to State-of-the-art Machine learning models}\label{sec:imagenet}
Here we want to further test the high frequency bias of bounded width convolutional models in models trained on one of the most complex dataset of image recognition: Imagenet~\cite{Imagenet}.


We selected different models from the timm package \cite{timm}, that are representative of complex models with different model parametrizations. All the models were selected from 4 groups: Convolution-Based Models, Vision Transformers, Hybrid ViT and Convolutional Models, and MLPs models (see Supp Section~\ref{sec:architectures} for specifics pretrained timm models). All selected models have similar data augmentation procedures since data augmentation can affect the frequency sensitivity of a model~\cite{li2022robust,hermann2020origins, yin2019fourier}.

Figure \ref{fig:transformer_imagenet} shows the radial integral of the energy of the adversarial perturbations of the ImageNet validation set: we observe that the convolutional-based models (Blue) have less energy in the low frequencies and more in the high frequencies compared to all other models. As for the CIFAR10 dataset, the ViT models (Purple) have lower frequency adversarial. More interestingly, models with larger first layer kernels (ViT (32)) have more energy in the low frequencies compared to models with smaller kernel sizes (ViT(16), ViT(8)). This confirms the results in section \ref{sec:nonconv} with models trained on CIFAR10.  Interestingly, the work in \cite{park2022vision} showed that self-attention layers can produce models with lower frequency preference and that a combination of the Vision Transformers with Convolutions generates a compromise frequency preference. Here we confirm those findings showing that hybrid models already have a $f_{50}$ in between ViTs and Convolution models and also find that also MLP-Based models have similar energy distribution compared to hybrid models. This indicates that self-attention is not the only mechanism to achieve a low frequency bias, and that the linear layer parametrization seem to be able to achieve similar effects. All together these results show that, regardless of the dataset, convolution-based models have a preferences towards higher frequency features compared to the non-convolutional counterparts.
\section*{Conclusions}\label{sec:concl-fw}

We provide both empirical as well as theoretical evidence that one of the causes of the high-frequency adversarial examples is the convolutional architecture of modern high-performance networks and, in particular, the locality/boundness of the convolutional kernels. To this end, we first confirmed the theoretical results regarding the network implicit bias in deep linear full-width convolutional models~\cite{gunasekar2018implicit}, extend them to nonlinear models (section \ref{sec:fullwidth}) and correlate the end-to-end weights of the model with the adversarial perturbation.


We then considered deep linear bounded-width convolutions and derived new theoretical results showing a bias towards high-frequency features compared to other model parametrizations in section~\ref{sec:boundedwidth}. 


In section \ref{sec:translinv} we showed that convolutional weights sharing is a key factor in producing the high frequency bias in networks.

Moreover, by a novel steganography experiment, we could clearly confirm that this bias extends to nonlinear models, and that it strongly influences what features can be easily learned by the models (section \ref{sec:shortcutfeats}). We think is a new and exciting directions to probe the ability of the model to generalize along any basis not just Fourier. In addition, our work has shown that adversarial perturbations can be an important technique to analyze the features learned by a neural networks at the end or during training. 

Furthermore, by using different datasets, we were able to show that convolution-based models have more energy in the high frequencies \textbf{if} high frequency useful information is present. This results deviate from the general understanding of adversarial perturbations and generalization, where high frequency adversarial attacks are only determined by the dataset statistics. Instead, we show that even in datasets with higher frequency information such as CIFAR10, CIFAR100 and SVHN, models with non-convolutional architectures (e.g. fully connected, locally connected, and Vision Transformers) do not present high frequency adversarial attacks (section \ref{sec:nonconv}). In addition we demonstrated that ViTs with smaller kernel sizes have more energy in the high frequencies for both CIFAR10 and ImageNet trained models, suggesting  that these results are robust regardless of dataset statistics. Moreover removing the convolutional layer in ViTs (Local Vision Transformer Architecture), allow us to reduced the energy concentration of Transformer models even further. This shows that one could potentially analyze the implicit bias of different architectural motifs in simple models and extract those principles into more complex scenarios.

However, recent work has shown that self attention does induce a low frequency bias in ViTs \cite{park2022vision}. They showed that self-attention layers have a low pass filtering effect on the features received from the previous layer and therefore this is why ViTs are more focused on low-frequency features \cite{park2022vision}. Here we show that other non-convolutional architectures such as MLPMixer, ResMLP and gMixer have lower frequency adversarial attacks similar to ViTs even without self-attention layers. This shows that self-attention is not the only mechanism to achieve low-frequency biased models, and that the linear layer parametrization also has a significant effect in this behavior. We think this is important because ViTs, MLPs or Convolution-based models such as ResNets, have multiple components such as Norm Layers, Skip Connections, Attention, Layer Mixing, etc, that may overpower the linear parametrization in defining the implicit bias of the model. However, we found that the linear layer parametrization seem to be crucial to the nature of the features learned and therefore for the adversarial perturbations. 

Finally, much work about generalization and adversarial attacks has been centered around bias in datasets~\cite{wang2020high,ilyas2019adversarial}, however, here we show that attention needs to be paid to the bias caused by architectural choices as well. This work is just the start of establishing the relationship between implicit bias due to model parametrization and the robustness of neural networks to different kinds of input perturbations. We believe this understanding can help drive models towards the most useful set of features and reduce the adversarial susceptibility of modern neural networks.

\section*{Acknowledgments}

This research has been funded by the NSF NeuroNex program through grant DBI-1707400. This research was also supported by Intelligence Advanced Research Projects Activity (IARPA) via Department of Interior/Interior Business Center (DoI/IBC) contract number D16PC00003. The U.S. Government is authorized to reproduce and distribute reprints for Governmental purposes notwithstanding any copyright annotation thereon. Disclaimer: The views and conclusions contained herein are those of the authors and should not be interpreted as necessarily representing the official policies or endorsements, either expressed or implied, of IARPA, DoI/IBC, or the U.S. Government.

\section*{Acknowledgments}

This research has been funded by the NSF NeuroNex program through grant DBI-1707400. This research was also supported by Intelligence Advanced Research Projects Activity (IARPA) via Department of Interior/Interior Business Center (DoI/IBC) contract number D16PC00003. The U.S. Government is authorized to reproduce and distribute reprints for Governmental purposes notwithstanding any copyright annotation thereon. Disclaimer: The views and conclusions contained herein are those of the authors and should not be interpreted as necessarily representing the official policies or endorsements, either expressed or implied, of IARPA, DoI/IBC, or the U.S. Government.
\section{Methods}\label{sec:methods}

\subsection{Neural Network Training.} Each architecture was defined by the number of hidden layers, and nonlinearities (See Supp Table~\ref{tbl:architectures}). In terms of hyperparameters, we tuned the maximum learning rates for each model by starting from a base learning rate of $0.1$, and then, if there were visible failures during training (most commonly, the model converging to chance performance), we adjusted the learning rate up/down by a factor of $10$ or $50$. Amongst the model architectures we explored, the only hyper-parameter that was tuned was the learning rate. The final values of the learning rates after search are detailed in Table~\ref{tbl:lr}. In addition, all the models were trained with linearly decaying learning rate follow $0.3$ factor for each epoch and resetting the learning rate back to max when the model was trained at least $20$ epochs. All the models were trained on a single GTX $1080$ Ti for at least $40$ epochs  ($30$ to $120$ GPU minutes), and we choose the epoch with the highest validation set accuracy for further experiments (see hyperparameters of training and accuracy on Supp Sec.~\ref{sec:supp-training-configs})

\subsection{Adversarial Attack Generation.} We used the Foolbox package \cite{rauber2017foolbox} (MIT license) to generate adversarial perturbations $\delta$ for every example in the test set for a fully trained model (PGD-Linf, PGD-L2, PGD-L1\cite{kurakin2016adversarial}, BB-Linf and BB-L2 \cite{brendel2019accurate}). Finally, we computed the 2-D Discrete Fourier spectrum $\hat{\delta} := \mathcal{F}\delta$ of the perturbation $\delta$. Details of the attacks are available in Supp. Table~\ref{tbl:hyper_attack}.

\subsection{$\beta$ Calculation and Toeplitz Matrix.} 

For the computation of $\beta$, we used two different methods. For the linear models, we transformed the weights of every architecture into their matrix form. For example for the convolutional operation, we generated a Toeplitz matrix per convolutional filter and then calculated  the dot product of the first $l$ matrices to get the $\beta_{l}$  (or for all $l=1,\cdots,L$ to get the input-output function $\beta$). For the nonlinear models, because the nonlinearities does not allow us to use the weights directly, we decided to use a proxy, the saliency map. Saliency Map is the gradient ($\frac{df}{dx}$) of the function ($f(x)$) with respect to the input image ($x$). In the linear case, these gradients are exactly the weights of the function $\beta$ (up to a constant), which we confirmed using the Toeplitz computation above. For the nonlinear models, because the weights used changed per example, the gradient gave us a good approximation of those weights.

\subsection{Generation of Hidden Shortcut Features.} For each class, we sampled a $3 \times 32 \times 32$ matrix of scalars from a standard Gaussian distribution with mean of 0 and standard deviation of 1 ($N_{class}$). Then we multiplied this matrix by a scalar factor $\epsilon$, which was selected via hyperparameter search. Next, we generated a masking matrix ($M_{class}$) of the same size, $3 \times 32 \times 32$. For the sparse shorcut feature experiment this matrix was generated from a Bernoulli distribution with an specific average sparsity level $\rho$. For the frequency-based class dependent features, we selected  an upper and lower bound radii as a factor of the max frequency ($0.1$ and $0.2$; $0.4$ and $0.5$; and $0.8$ and $0.9$ for the low, medium and high frequencies respectively). These were the ranges used to produce the masking matrix $M_{class}$. Next, we computed the hardmard product of $M_{class}$ with $N_{class}$. Then, this class-specific shortcut feature was added into the Fourier spectrum of CIFAR-10 train and test images corresponding to their respective classes. The mathematical definitions are as follows:

\begin{align*}
     N_{\textrm{class}} &\sim \mathcal{N}(0, 1) \\
     M_{class} &= \begin{cases} \mbox{Sparse} & M_{\textrm{class}} \sim \textrm{Bernoulli}(\rho) \\ \mbox{Frequency} & M_{\textrm{class}} = \begin{cases} \mbox{Inside frequency range} & $1$ \\ \mbox{Outside frequency range} & $0$\\ \end{cases}\end{cases} \\
     \eta_{\textrm{class}} &= M_{\textrm{class}} \odot (\epsilon * N_{\textrm{class}}),\;\;\;\;\;\;  \hat{\tilde{x}} = \hat{x} + \eta_{\textrm{class}}. 
\end{align*}

\medskip
\bibliography{neurips2021}
\bibliographystyle{unsrtnat}

\appendix
\newpage
\onecolumn
\section{Supplementary Material}\label{sec:supp-extras}

\subsection{Radial energy distribution for different Datasets}\label{sec:different_datasets}

\begin{suppfigure}[h]
    \centering
    \includegraphics[width=0.9\textwidth]{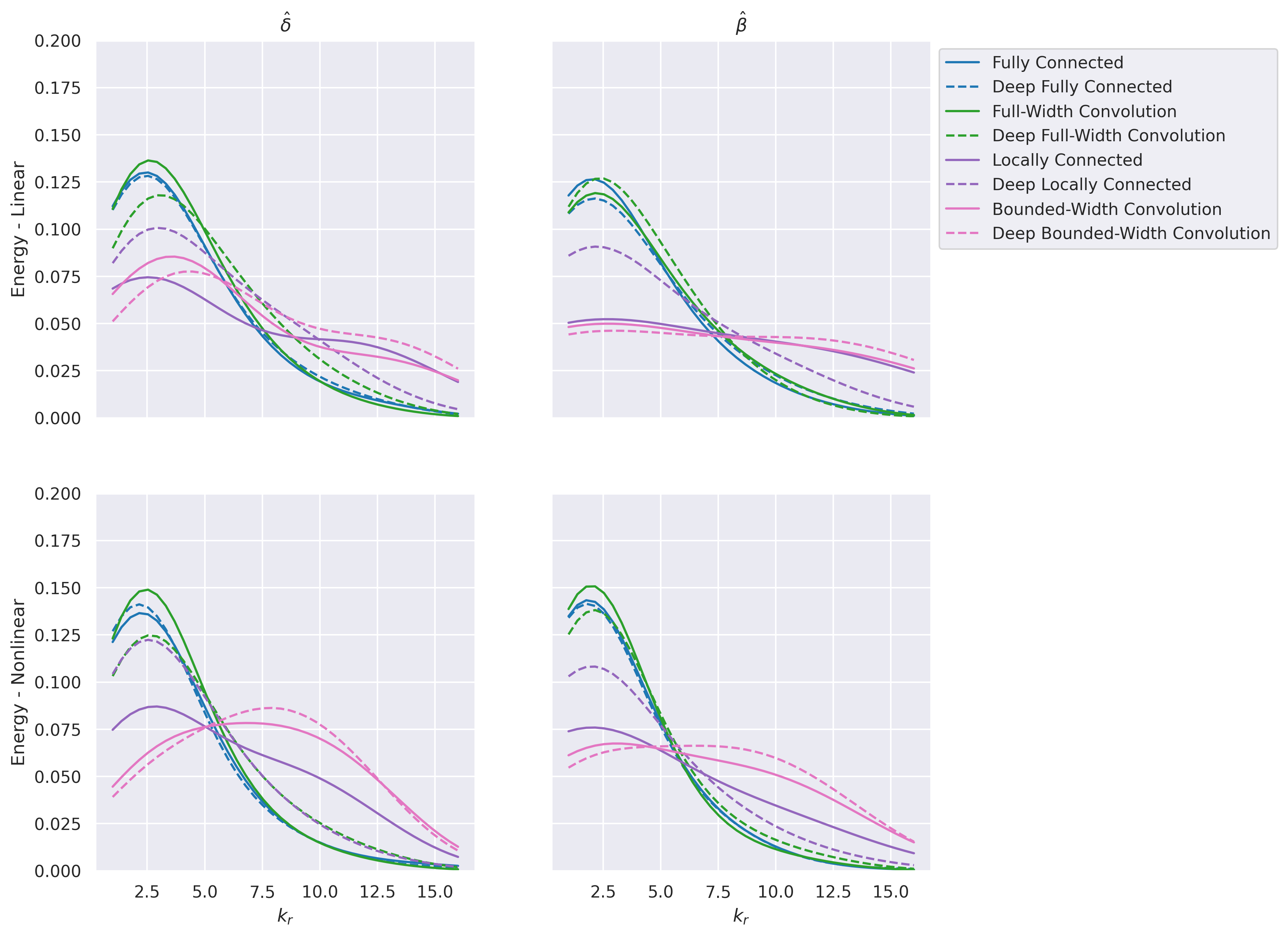}
    \caption{Radial Integral of the average frequency spectrum of adversarial perturbation ($\hat{\delta}$) and input-output weights ($\hat{\beta}$) for Fully Connected (blue), Full-Width Convolutional (green), Locally Connected (purple) and Bounded-Width Convolutional (pink) models for PGD-Linf Attack for CIFAR10. We can observe the correlation between the input-output weights and the adversarial perturbations for each model.}
    \label{fig:attacks_cifar10_1D}
\end{suppfigure}

\begin{suppfigure}[h]
    \centering
    \includegraphics[width=0.9\textwidth]{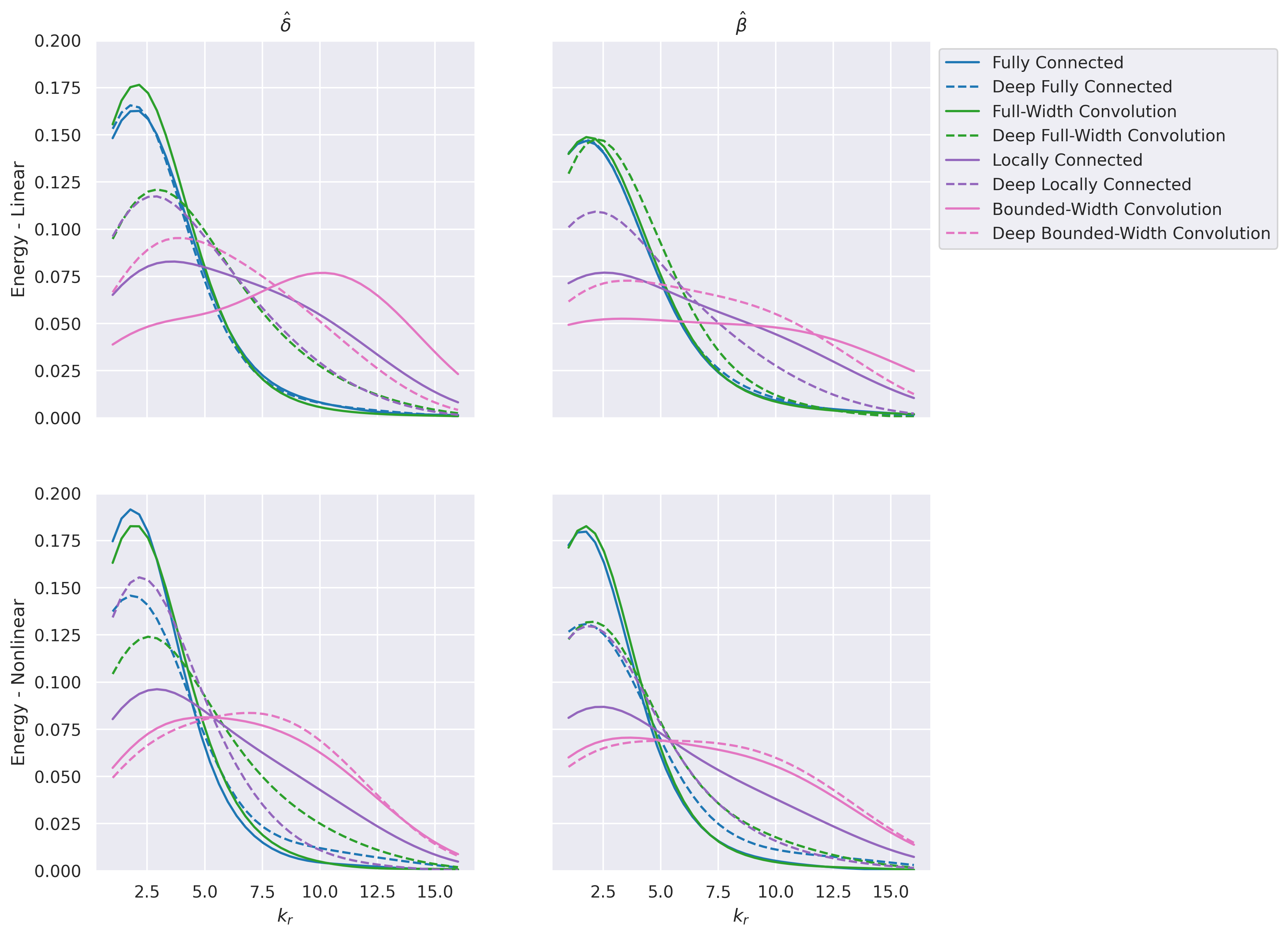}
    \caption{Radial Integral of the average frequency spectrum of adversarial perturbation ($\hat{\delta}$) and input-output weights ($\hat{\beta}$) for Fully Connected (blue), Full-Width Convolutional (green), Locally Connected (purple) and Bounded-Width Convolutional (pink) models for PGD-Linf Attack for CIFAR100. We can observe the correlation between the input-output weights and the adversarial perturbations for each model.}
    \label{fig:attacks_cifar100_1D}
\end{suppfigure}

\begin{suppfigure}[h]
    \centering
    \includegraphics[width=0.9\textwidth]{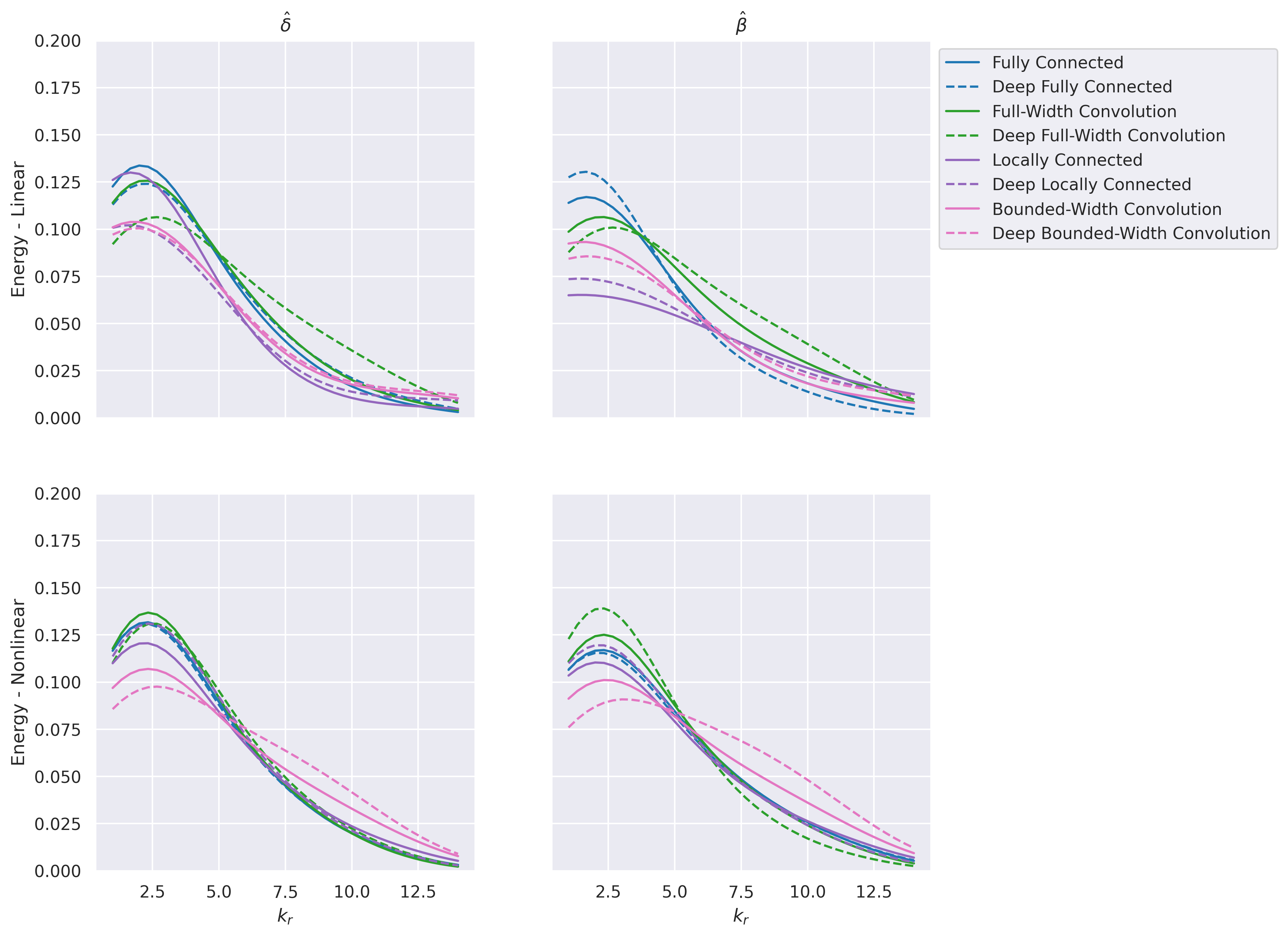}
    \caption{Radial Integral of the average frequency spectrum of adversarial perturbation ($\hat{\delta}$) and input-output weights ($\hat{\beta}$) for Fully Connected (blue), Full-Width Convolutional (green), Locally Connected (purple) and Bounded-Width Convolutional (pink) models for PGD-Linf Attack for FashionMNIST. We can observe the correlation between the input-output weights and the adversarial perturbations for each model.}
    \label{fig:attacks_FashionMNIST_1D}
\end{suppfigure}

\begin{suppfigure}[h]
    \centering
    \includegraphics[width=0.9\textwidth]{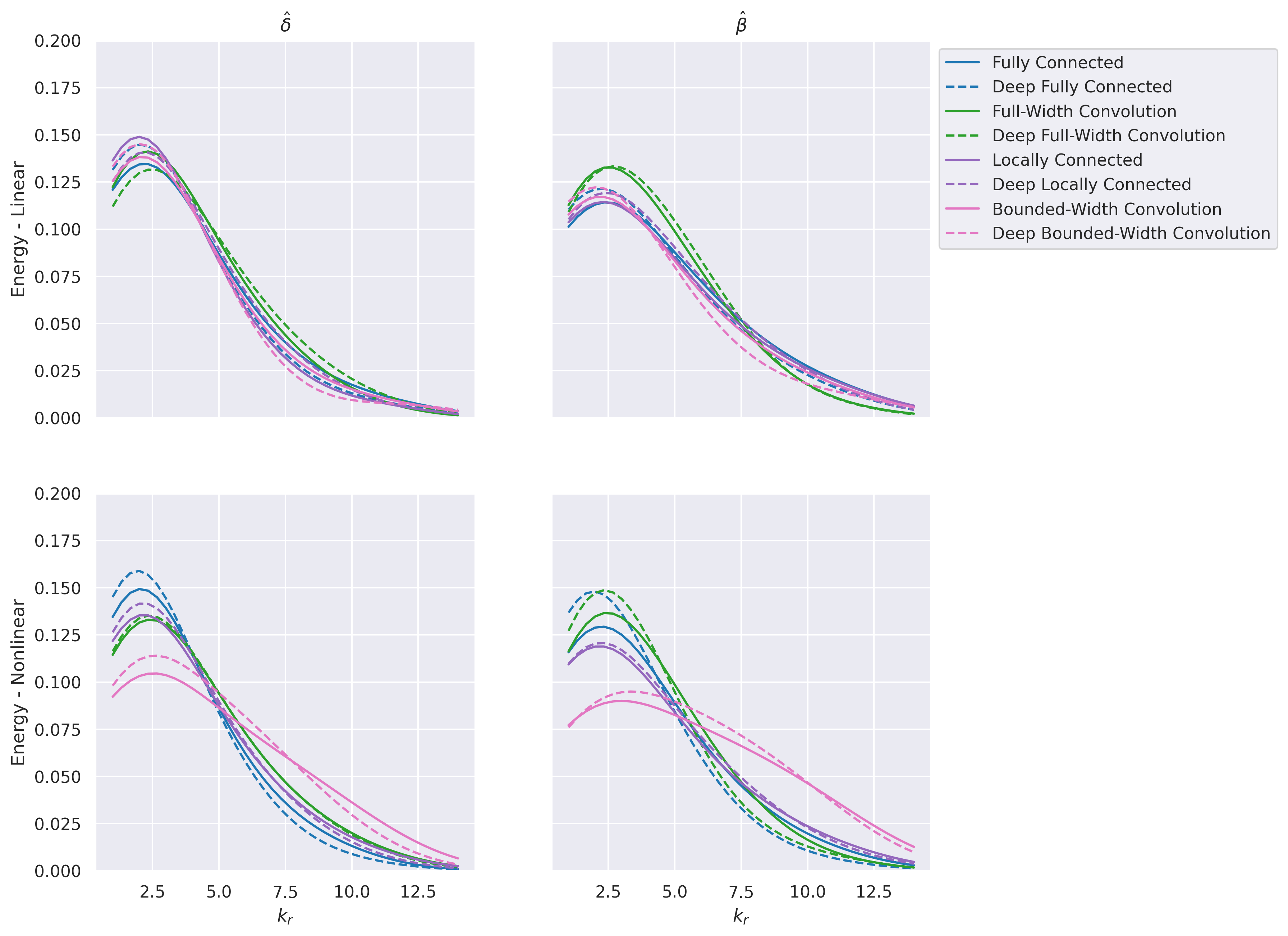}
    \caption{Radial Integral of the average frequency spectrum of adversarial perturbation ($\hat{\delta}$) and input-output weights ($\hat{\beta}$) for Fully Connected (blue), Full-Width Convolutional (green), Locally Connected (purple) and Bounded-Width Convolutional (pink) models for PGD-Linf Attack for MNIST. We can observe the correlation between the input-output weights and the adversarial perturbations for each model.}
    \label{fig:attacks_MNIST_1D}
\end{suppfigure}

\begin{suppfigure}[h]
    \centering
    \includegraphics[width=0.9\textwidth]{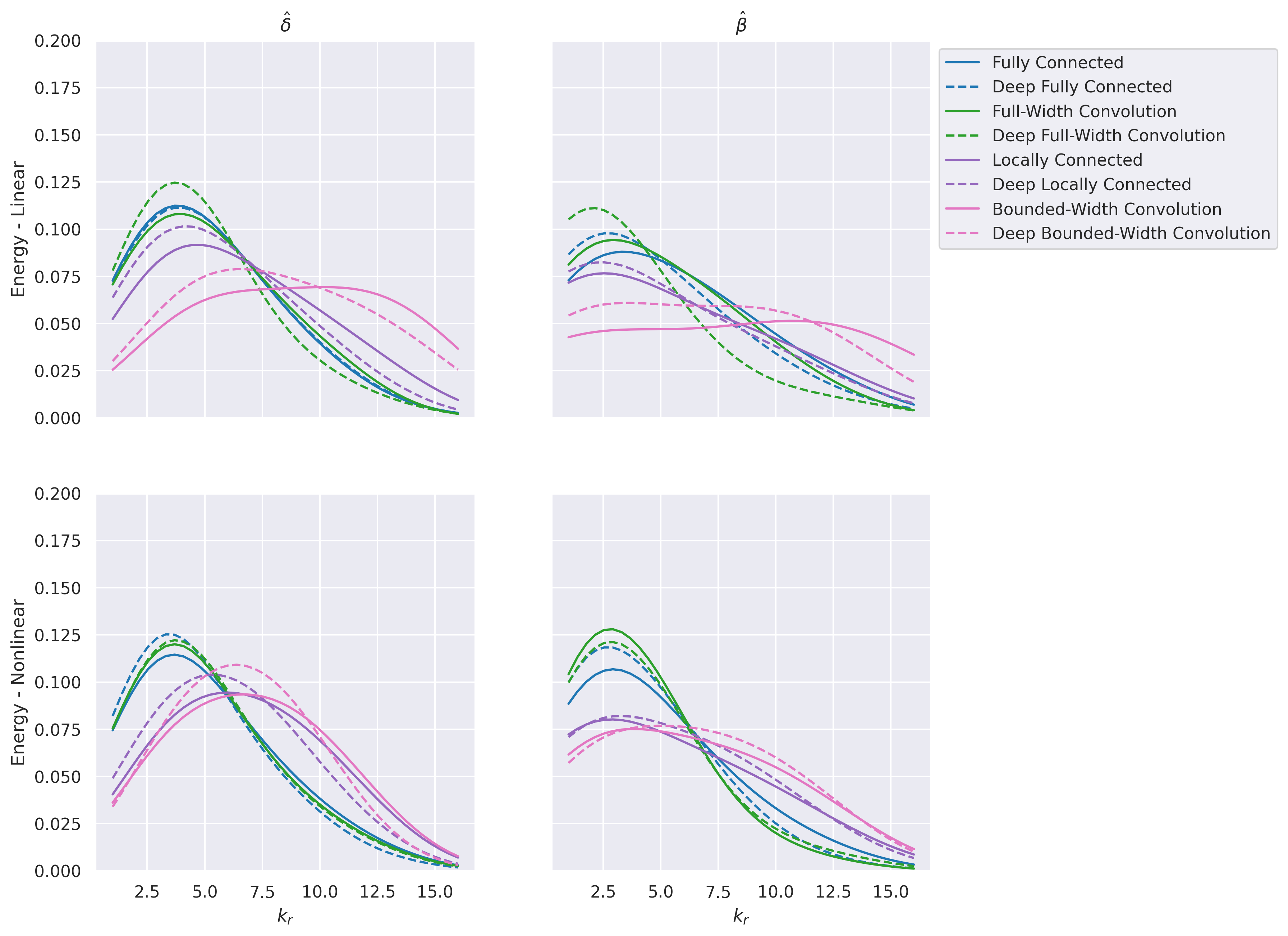}
    \caption{Radial Integral of the average frequency spectrum of adversarial perturbation ($\hat{\delta}$) and input-output weights ($\hat{\beta}$) for Fully Connected (blue), Full-Width Convolutional (green), Locally Connected (purple) and Bounded-Width Convolutional (pink) models for PGD-Linf Attack for SVHN. We can observe the correlation between the input-output weights and the adversarial perturbations for each model.}
    \label{fig:attacks_SVHN_1D}
\end{suppfigure}

\begin{suppfigure}[h]
    \centering
    \includegraphics[width=0.9\textwidth]{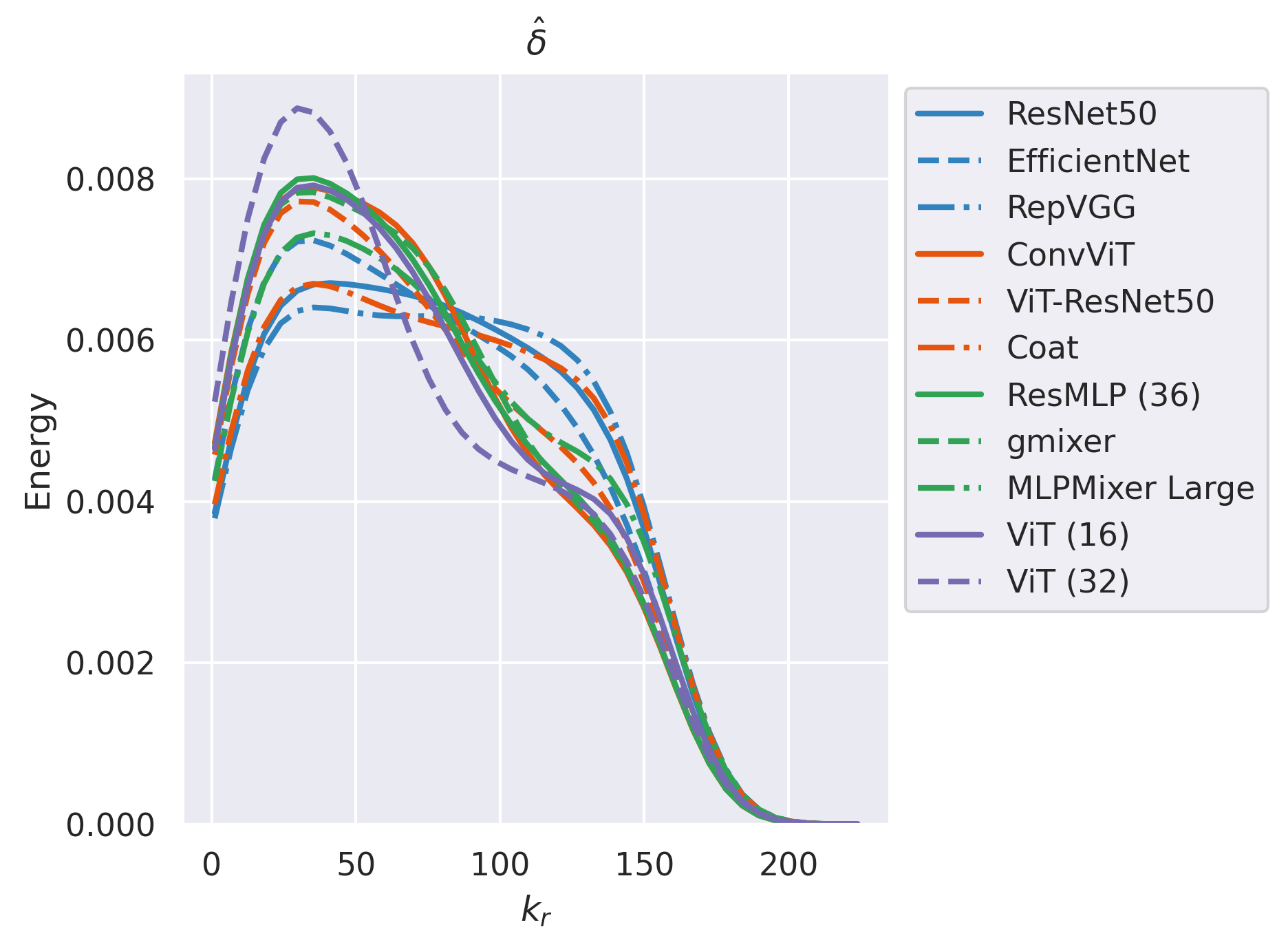}
    \caption{Radial Energy of the Adversarial Perturbations of Convolutional (Blue), Hybrid Vision Transformer and Convolutional Model (Orange), MLPMixer Variants (Green) and Vision Transformer (purple) models. All the plots are computed over the 50000 validation set images. Convolutional Models have the most energy in the high frequencies across all models.}
    \label{fig:attacks_imagenet_1D}
\end{suppfigure}



\begin{figure*}[h]
\begin{subfigure}{\textwidth}
  \centering
  \includegraphics[trim={0 0 0 0},clip,width=0.8\textwidth]{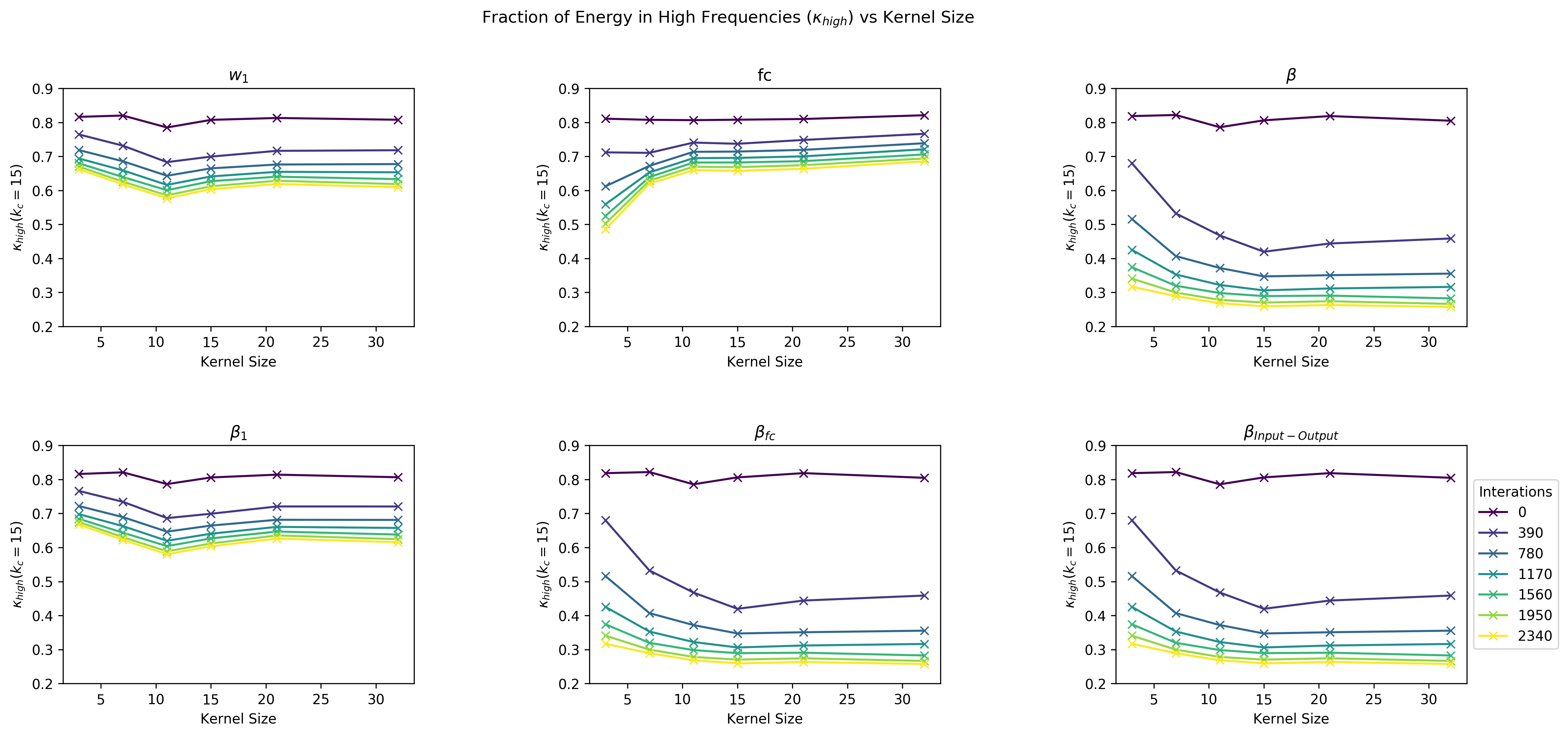}
  \caption{}\label{fig:single}
\end{subfigure}
\hfill
\begin{subfigure}{\textwidth}
  \centering
  \includegraphics[trim={0 0 0 0},clip,width=0.8\textwidth]{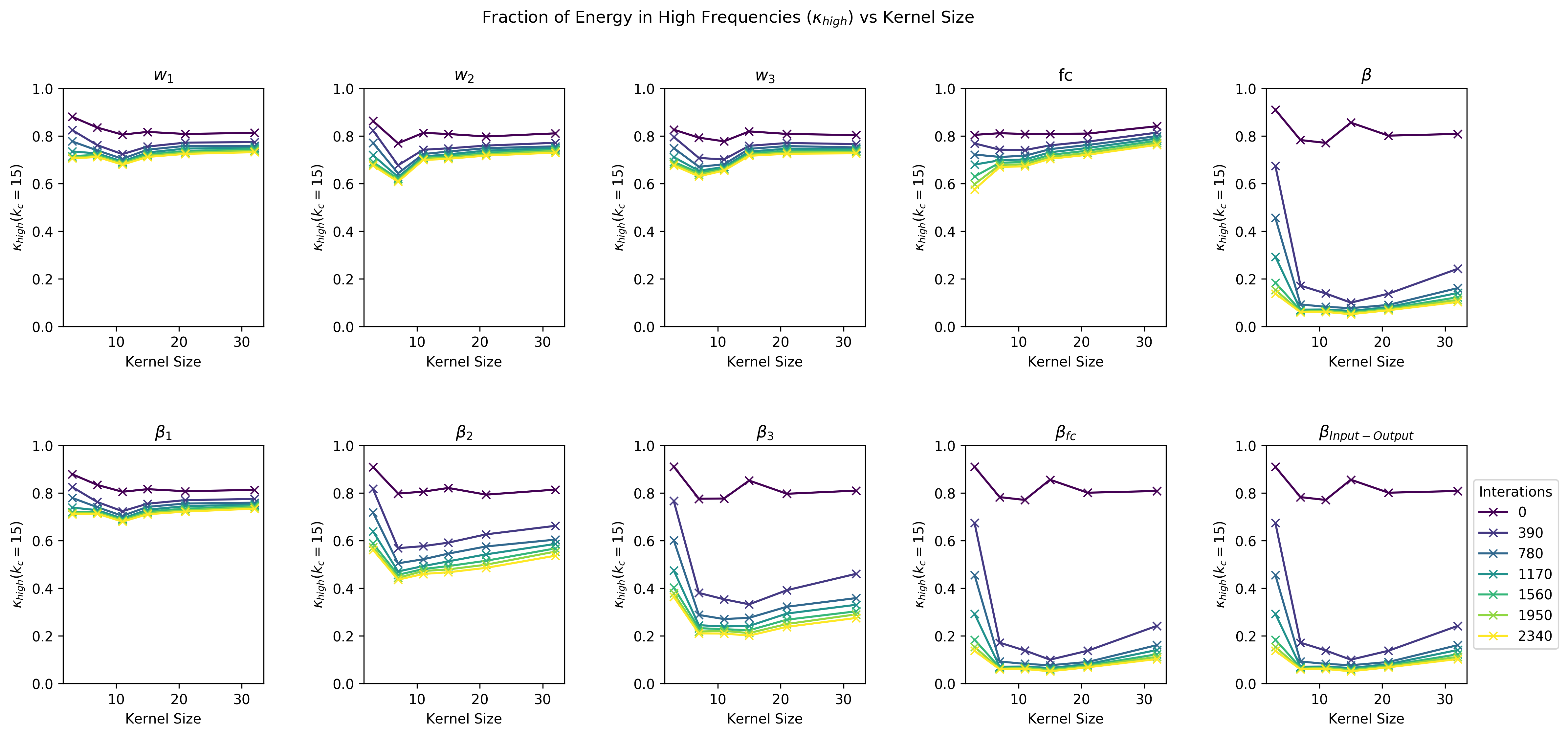}
  \caption{}\label{fig:multi}
\end{subfigure}
  \caption{Concentration of energy $\kappa_{high}$ for $\kappa_{r}$ = 15 for each $\beta_{l}$ for all $l$ versus kernel size (3,7,9,11,15, 21, or 32). All models were trained for 40 epochs on Grayscale CIFAR10. (a) One hidden layer models (b) Three hidden layer models.  }\label{fig:kappa} 
\end{figure*}

\begin{suppfigure}[h]
    \centering
    \includegraphics[width=0.9\textwidth]{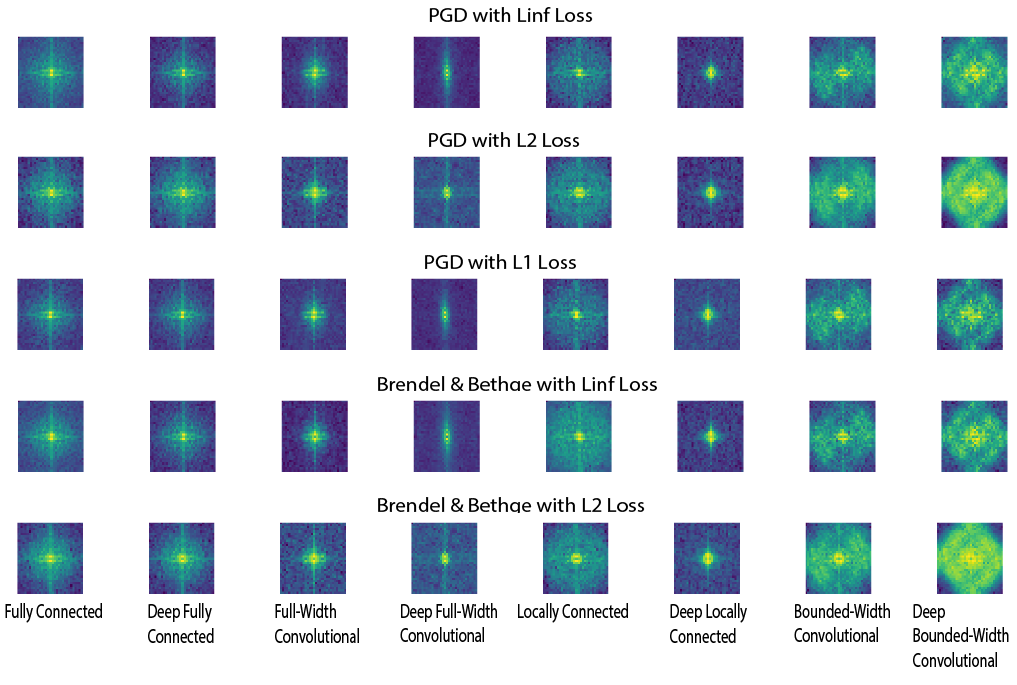}
    \caption{Average Adversarial Perturbation Fourier Spectrum for Fully Connected, Full-Width, Locally Connected and Bounded-Width \textbf{Linear} models.}
    \label{fig:attacks_cifar10}
\end{suppfigure}

\begin{suppfigure}[h]
    \centering
    \includegraphics[width=0.9\textwidth]{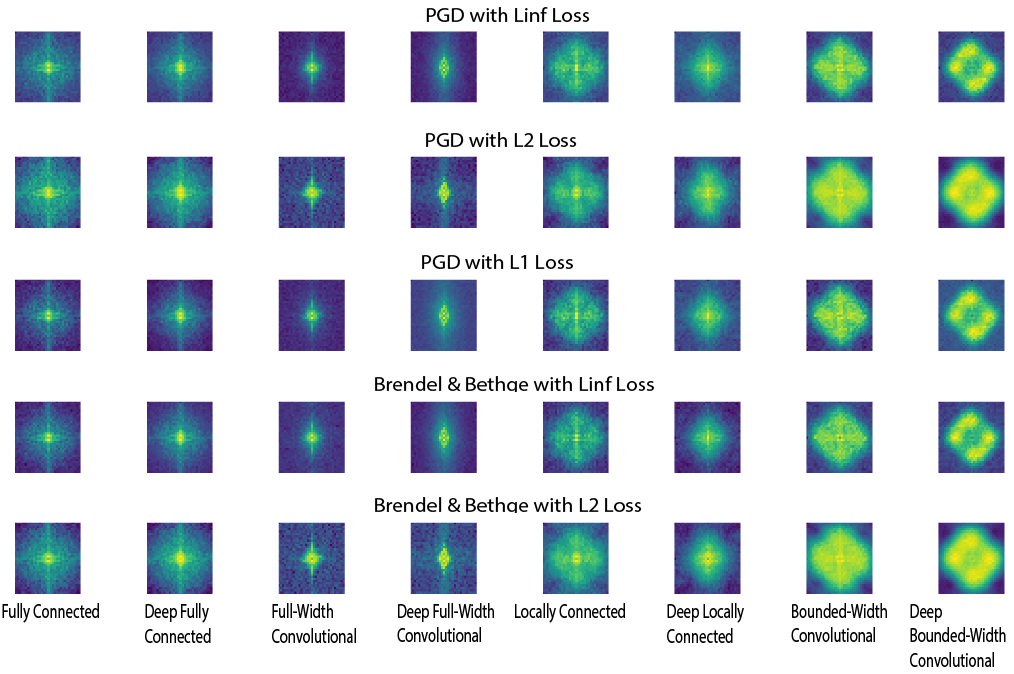}
    \caption{Average Adversarial Perturbation Fourier Spectrum for Fully Connected, Full-Width, Locally Connected and Bounded-Width \textbf{NonLinear} models.}
    \label{fig:attacks_cifar10_relu}
\end{suppfigure}

\clearpage
\subsection{Tables}

        


\begin{supptable}[h]
\small
\centering
\caption{Frequency-Domain Norms of Predictor $\hat{\beta}$ and Adv. Pert. $\hat{\delta}$ for Linear Models. With Equivalent performance (See Sup Figure~\ref{tbl:accuracy}), Convolutional Models have smaller $\| \hat{\beta} \|_1$, supporting the implicit Fourier hypothesis.}
\begin{adjustbox}{max width=\textwidth}
\begin{tabular}{|l| c| c| c| c|}
\hline
\textbf{Model} & $\| \hat{\delta} \|_2$ &$\| \hat{\delta} \|_1$ & $\| \hat{\beta} \|_2$ & $\| \hat{\beta} \|_1$\\
\hline

Fully Connected & 20.77 $\pm$ 1.23 & 505.95 $\pm$ 32.22 & 546.89 $\pm$ 43.37 & 2232.46 $\pm$ 166.27 \\
\hline
Full Width Convolution & 17.52 $\pm$ 1.56 & 376.31 $\pm$ 32.66 & 544.83 $\pm$ 42.40 & 1598.37 $\pm$ 181.44 \\
\hline
Bounded Width Convolution & 20.65 $\pm$ 1.09 & 556.32 $\pm$ 32.64 & 500.67 $\pm$ 47.46 & 2623.04 $\pm$ 331.42 \\
\hline
Deep Fully Connected & 20.83 $\pm$ 1.25 & 503.92 $\pm$ 32.62 & 554.16 $\pm$ 38.95 & 2168.06 $\pm$ 134.07 \\
\hline
Deep Full Width Convolution & 17.13 $\pm$ 1.68 & 265.40 $\pm$ 41.72 & 522.99 $\pm$ 39.72 & 1417.68 $\pm$ 213.91 \\
\hline
Deep Bounded Width Convolution & 20.57 $\pm$ 0.94 & 550.30 $\pm$ 29.71 & 498.52 $\pm$ 47.15 & 2392.09 $\pm$ 278.55 \\
\hline
\end{tabular}
\end{adjustbox}
\label{tbl:linear_models}
\end{supptable}

\begin{supptable}[h]
\small
\centering
\caption{Frequency-Domain Norms of Adv. Pert. $\hat{\delta}$ for Nonlinear Models. With similar performance (Supp. Table~\ref{tbl:linear_models}), Full Width and Bounded Width convolutional models produce adversarial attacks with smaller $\| \hat{\delta} \|_1$.}
\begin{tabular}[h]{|l| c| c|}
\hline
\textbf{Model} & $\| \hat{\delta} \|_2$ &$\| \hat{\delta} \|_1$ \\
\hline
Fully Connected & 22.06 $\pm$ 1.35 & 534.34 $\pm$ 35.93 \\
\hline
Locally Connected & 22.32 $\pm$ 1.13 & 587.61 $\pm$ 29.89 \\
\hline
Full Width Convolution & 20.57 $\pm$ 1.11 & 405.49 $\pm$ 34.16 \\
\hline
Bounded Width Convolution & 24.63 $\pm$ 1.36 & 665.22 $\pm$ 38.55 \\
\hline
Shallow ResNet & 26.36 $\pm$ 0.74 & 709.06 $\pm$ 20.76 \\
\hline
VGG19 & 23.32 $\pm$ 0.99 & 650.73 $\pm$ 27.87 \\
\hline
ResNet18 & 22.42 $\pm$ 0.91 & 626.60 $\pm$ 25.40 \\
\hline
\end{tabular}

\label{tbl:nonlinear_models}
\end{supptable}

\clearpage
\section{Experimental Details}\label{sec:supp-configs}

\subsection{Model Architecture}\label{sec:architectures}

\begin{supptable}[h]
\begin{center}
\caption{Model Architectures}\label{tbl:architectures}
\begin{adjustbox}{max width=\textwidth}
\begin{tabular}{|c| c| c| c|}
\hline
    \textbf{Model Architecture} & \textbf{$\#$ Hidden Layers} & \textbf{Nonlinearity} & \textbf{Channels} \\
    \hline
    Fully Connected & 1,3 & None, ReLU & 3072\\
    Bounded Width Convolution & 1,3 & None, ReLU & 32\\
    Full Width Convolution & 1,3 & None, ReLU & 32\\
    Locally Connected & 1,3 & None, ReLU & 32\\
    \hline

\end{tabular}
\end{adjustbox}
\end{center}
\end{supptable}

\begin{supptable}[H]
\begin{center}
\caption{Model Configurations for ImageNet Trained Models. All models were pulled from the timm package.}
\label{tbl:accuracy}
\begin{tabular}{|c| c|c|}
\hline
    \textbf{Model} & \textbf{Model Type} &\textbf{Model timm package name} \\
    \hline
    ResNet50 \cite{} & Convolutional & resnet50d\\
    EfficientNet \cite{} & Convolutional & tf\_efficientnet\_b0\_ns\\
    RepVGG \cite{} & Convolutional & repvgg\_b3\\
    ConvViT \cite{} & Hybrid & convit\_base\\
    ViT-ResNet50 \cite{} & Hybrid & vit\_large\_r50\_s32\_224\\
    Coat \cite{} & Hybrid & coat\_lite\_small\\
    ResMLP (36) \cite{} & MLP & resmlp\_36\_distilled\_224\\
    gMixer \cite{} & MLP & gmixer\_24\_224 \\
    MLPMixer Large \cite{} & MLP & mixer\_b16\_224 \\
    ViT (8) \cite{} & ViT & vit\_base\_patch8\_224 \\
    ViT (16) \cite{} & ViT & vit\_base\_patch16\_224\\
    ViT (32) \cite{} & ViT & vit\_base\_patch32\_224 \\
    \hline
\end{tabular}
\end{center}
\end{supptable}

All ImageNet models were pull from the timm package \cite{}. Furthermore, all models were trained with similar data augmentations, and adversarial attack evaluation was done with default preprocessing from $model.default_cfg$.

\subsection{Model Performance} \label{sec:supp-performance}

\begin{supptable}[H]
\begin{center}
\caption{Test Accuracy for all models trained on CIFAR-10, CIFAR100, MNIST, FashionMNIST, SVHN.}
\label{tbl:accuracy}
\begin{adjustbox}{max width=\textwidth}
\begin{tabular}{llllll}
\toprule
{}              &   \multicolumn{5}{c}{Test Accuracy (\%)}                \\
\midrule
Models &  FashionMNIST &  MNIST &  SVHN &  cifar10 &  cifar100 \\
\midrule

Fully Connected                       &          86.9 &   92.0 &  26.5 &     39.6 &      15.8 \\
Full-Width Convolution                &          86.6 &   91.0 &  28.7 &     40.5 &      17.7 \\
Locally Connected                     &          86.4 &   92.0 &  28.5 &     40.7 &      18.3 \\
Bounded-Width Convolution             &          86.4 &   92.0 &  28.2 &     40.2 &      16.1 \\
Deep Full-Width Convolution           &          86.5 &   91.8 &  26.0 &     41.7 &      18.9 \\
Deep Fully Connected                  &          86.7 &   92.1 &  23.8 &     39.2 &      14.7 \\
Deep Locally Connected                &          84.1 &   92.0 &  29.0 &     41.8 &      18.8 \\
Deep Bounded-Width Convolution        &          86.1 &   92.0 &  27.7 &     39.9 &      14.9 \\
\midrule
Fully Connected (ReLU)                &          88.4 &   97.2 &  77.6 &     43.6 &      15.4 \\
Full-Width Convolution (ReLU)         &          84.8 &   92.8 &  85.9 &     47.7 &      19.2 \\
Locally Connected (ReLU)              &          87.8 &   95.7 &  82.4 &     53.8 &      22.5 \\
Bounded-Width Convolution (ReLU)      &          90.5 &   98.0 &  83.2 &     59.7 &      29.0 \\
Deep Full-Width Convolution (ReLU)    &          87.8 &   97.4 &  86.3 &     51.7 &      22.2 \\
Deep Fully Connected (ReLU)           &          89.1 &   97.9 &  67.3 &     50.8 &      10.5 \\
Deep Locally Connected (ReLU)         &          86.5 &   96.0 &  86.9 &     57.9 &      18.3 \\
Deep Bounded-Width Convolution (ReLU) &          91.7 &   98.8 &  86.7 &     66.0 &      29.6 \\
\midrule
ViT-2 &  & & &79.8 &\\
ViT-4 &  & & &80.1 &\\
ViT-8 & & & & 78.4 &\\
ViTLoc-2 &  & & &75.3 &\\
ViTLoc-4 &  & & &80.2 &\\
ViTLoc-8 &  & & &74.5 &\\
ResNet18  &  & & &91.8 &\\
\bottomrule
\end{tabular}
\end{adjustbox}
\end{center}
\end{supptable}

\begin{supptable}[H]
\begin{center}
\caption{Test Accuracy for all models trained on ImageNet.}
\label{tbl:accuracy}
\begin{tabular}{lrr}
\toprule
                  model &   top1 &   top5 \\
\midrule
    vit\_base\_patch8\_224 & 85.794 & 97.794 \\
   vit\_base\_patch16\_224 & 84.528 & 97.294 \\
  vit\_large\_r50\_s32\_224 & 84.424 & 97.166 \\
        coat\_lite\_small & 82.304 & 95.848 \\
            convit\_base & 82.286 & 95.938 \\
resmlp\_36\_distilled\_224 & 81.154 & 95.488 \\
   vit\_base\_patch32\_224 & 80.722 & 95.566 \\
              resnet50d & 80.522 & 95.162 \\
              repvgg\_b3 & 80.496 & 95.264 \\
  tf\_efficientnet\_b0\_ns & 78.658 & 94.378 \\
          gmixer\_24\_224 & 78.036 & 93.670 \\
          mixer\_b16\_224 & 76.612 & 92.228 \\
\bottomrule
\end{tabular}
\end{center}
\end{supptable}

\subsection{Training Hyperparameters} \label{sec:supp-training-configs}

\begin{supptable}[h]
\begin{center}
\caption{Learning rates for the various models considered on CIFAR-10. All other hyper-parameters were fixed.}
\label{tbl:lr}
\begin{adjustbox}{max width=\textwidth}
\begin{tabular}{|c| c| c| c|}
\hline
    \textbf{Model Architecture} & \textbf{Learning Rate} & \textbf{Batch Size} & \textbf{Learning Rate Drop} \\
    \hline
    Bounded Width Convolution  & .01 & 128 & Yes\\
    Fully Connected & .01 & 128 & Yes\\
    Locally Connected  & .01 & 128 & Yes\\
    Full Width Convolution  & .002 & 128 & Yes\\
    Bounded Width Convolution & .002 & 128 & Yes\\
    \hline

\end{tabular}
\end{adjustbox}
\end{center}
\end{supptable}

\begin{supptable}[H]
\begin{center}
\caption{Learning rates for the various models considered on CIFAR-100. All other hyper-parameters were fixed.}
\label{tbl:lr_cifar100}
\begin{adjustbox}{max width=\textwidth}
\begin{tabular}{|c| c| c| c|}
\hline
    \textbf{Model Architecture} & \textbf{Learning Rate} & \textbf{Batch Size} & \textbf{Learning Rate Drop} \\
    \hline
    Bounded Width Convolution  & .01 & 128 & Yes\\
    Fully Connected & .01 & 128 & Yes\\
    Locally Connected  & .01 & 128 & Yes\\
    Full Width Convolution  & .002 & 128 & Yes\\
    Bounded Width Convolution & .002 & 128 & Yes\\
    \hline

\end{tabular}
\end{adjustbox}
\end{center}
\end{supptable}

\begin{supptable}[h]
\begin{center}
\caption{Learning rates for the various models considered on MNIST. All other hyper-parameters were fixed.}
\label{tbl:lr_mnist}
\begin{adjustbox}{max width=\textwidth}
\begin{tabular}{|c| c| c| c|}
\hline
    \textbf{Model Architecture} & \textbf{Learning Rate} & \textbf{Batch Size} & \textbf{Learning Rate Drop} \\
    \hline
    Bounded Width Convolution  & .01 & 100 & Yes\\
    Fully Connected & .01 & 100 & Yes\\
    Locally Connected  & .01 & 100 & Yes\\
    Full Width Convolution  & .002 & 100 & Yes\\
    Bounded Width Convolution & .002 & 100 & Yes\\
    \hline

\end{tabular}
\end{adjustbox}
\end{center}
\end{supptable}

\begin{supptable}[h]
\begin{center}
\caption{Learning rates for the various models considered on FashionMNIST. All other hyper-parameters were fixed.}
\label{tbl:lr_mnist}
\begin{adjustbox}{max width=\textwidth}
\begin{tabular}{|c| c| c| c|}
\hline
    \textbf{Model Architecture} & \textbf{Learning Rate} & \textbf{Batch Size} & \textbf{Learning Rate Drop} \\
    \hline
    Bounded Width Convolution  & .01 & 100 & Yes\\
    Fully Connected & .01 & 100 & Yes\\
    Locally Connected  & .01 & 100 & Yes\\
    Full Width Convolution  & .002 & 100 & Yes\\
    Bounded Width Convolution & .002 & 100 & Yes\\
    \hline

\end{tabular}
\end{adjustbox}
\end{center}
\end{supptable}

\begin{supptable}[H]
\begin{center}
\caption{Learning rates for the various models considered on SVHN. All other hyper-parameters were fixed.}
\label{tbl:lr_cifar100}
\begin{adjustbox}{max width=\textwidth}
\begin{tabular}{|c| c| c| c|}
\hline
    \textbf{Model Architecture} & \textbf{Learning Rate} & \textbf{Batch Size} & \textbf{Learning Rate Drop} \\
    \hline
    Bounded Width Convolution  & .01 & 128 & Yes\\
    Fully Connected & .01 & 128 & Yes\\
    Locally Connected  & .01 & 128 & Yes\\
    Full Width Convolution  & .002 & 128 & Yes\\
    Bounded Width Convolution & .002 & 128 & Yes\\
    \hline

\end{tabular}
\end{adjustbox}
\end{center}
\end{supptable}

\subsection{Adversarial Attack Configurations}\label{sec:supp-attack-configs}

\begin{supptable}[h]
\begin{center}
\caption{Adversarial Attack hyperparameters for CIFAR10, SVHN, CIFAR100, MNIST and FashionMNIST}
\begin{adjustbox}{max width=\textwidth}
\begin{tabular}{|c| c| c| c| c|}
\hline
    \textbf{Attack} & \textbf{Metric} & \textbf{Learning Rate} & \textbf{Number of Steps} & Max Norm, $\epsilon$  \\
    \hline
    Projected Gradient Descent & $L_{\infty}$ & 0.1 & 1000  & 8.0/255.0\\
     \hline
    Projected Gradient Descent & $L_{2}$ & 0.1 & 1000  & 2.0\\
    \hline
    Projected Gradient Descent & $L_{1}$ & 0.1 & 200  & 0.1\\
    \hline
    Brendel-Bethge Attack & $L_{\infty}$ & 1e-03 & 1000  & -\\
    \hline
    Brendel-Bethge Attack  & $L_{2}$ & 1e-03 & 1000  & -\\
    \hline
\end{tabular}
\end{adjustbox}\label{tbl:hyper_attack}
\end{center}
\end{supptable}

\textbf{Learning Rates.} All the models adversarial attacks were generated using the configuration above with the Foolbox package \protect\cite{rauber2017foolbox}.

\section{Formal Proofs of High Frequency Bias} \label{theory}
\thmucp*
\begin{proof}
Let us first concentrate on a single convolutional filter $w_{l} \in \mathbb{R}^{D}$.
Given an arbitrary choice of frequency interval $\Omega := \{-k, \ldots, 0, \ldots, +k \}$ and space interval $S := \{-a, \ldots, 0, \ldots, +a  \}$, we want to prove that reducing the energy fraction in the complementary set $S^{c}$ implies that we must increase the energy fraction in the complementary set $\Omega^{c}$.
The result follows from a direct application of the Uncertainty Principle for finite-dimensional vector spaces, as shown e.g. in Ghobber-Jaming~\cite{ghobber2011uncertainty}.
In particular let $\hat{w} \in \mathbb{R}^D$ be the coefficients of the Discrete Fourier Transform (DFT) of a convolutional filter $w \in \mathbb{R}^D$.
From equation 1.2 in \cite{ghobber2011uncertainty} we have 
\begin{equation}\label{UP}
  \| w \|_{\ell^2(S^{c})}+ \|\hat{w}\|_{\ell^2(\Omega^{c})} \geq \| w \|_{2} C(S,\Omega)
\end{equation}
where $C(S,\Omega)$ is constant when the intervals $S,\Omega$ are fixed. 
Dividing both sides of the inequality by $\| w \|_{2}$ we have
\begin{equation*}
   \kappa(S^{c}) + \hat{\kappa}(\Omega^{c}) \geq \textrm{const}
\end{equation*}
where $\kappa(\mathcal{A}) := \| w \|_{\ell^2(\mathcal{A})} / \| w \|_2$ is the spatial energy concentration of $w$ in the index set $\mathcal{A}$ and $\hat{\kappa}(\mathcal{B})$ is the frequency energy concentration of $\hat{w}$ in the set $\mathcal{B}$. 
Thus increasing the energy concentration in the spatial interval $S$ will cause a decrease in $S^{c}$ and by the inequality above an increase in $\kappa (\Omega^{c})$. If we let $\Omega$ be an interval of `low' frequencies, then we conclude there will be an increase in the energy concentration in the `high' frequencies $\Omega^c$.\\

The reasoning above can be extended from a single convolutional filter to the full end-to-end weights vector, $\beta := \star_{l=1}^{L-1} w_l$,  as follows.
Note first that, using the convolution theorem, the Discrete Fourier transform of $\beta$ is the Hadamard product of the Discrete Fourier transforms of the per-layer weights $w_l$ i.e.
\begin{equation*}
  \hat{\beta}=\hat{w}_{L-1}\odot\cdots\odot \hat{w}_{1}.   
\end{equation*}
Let us consider the energy in a set of `low' frequencies $\Omega$:
\begin{equation*}
  \hat{\beta}_{\Omega}=\hat{w}_{L-1,\Omega}\odot\cdots\odot \hat{w}_{1,\Omega}.   
\end{equation*}
Taking the $\ell_{2}$ norm and invoking the inequality $\|a\odot b\|_{2}\leq \|a\|_{2}\|b\|_{2}$ a total of $L-1$ times we can then write 
\begin{equation*}
  \kappa(\Omega,\beta)\leq \prod_{l=1}^{L-1}\kappa(\Omega,w_{l}). 
\end{equation*}
Suppose now that, all else equal, we decrease the energy concentration in each spatial domain $S$ of the per-layer filters $w_{l}$. 
By the reasoning above this will increase the energy concentration in frequency domain in the interval $\Omega^{c}$ i.e. a decrease in $\kappa(\Omega,w_{l})$ for each layer $l$.
By the last inequality this will decrease $\kappa(\Omega,\beta)$, resulting in an increase in the energy concentration in the high frequencies ($\Omega^{c}$) for $\beta$.
\end{proof}

\begin{lemma}
Concentrating the kernel energy in spatial domain increases the implicit regularization term in the optimization in [Gunasekar]:
\begin{equation*}
\forall a’ < a: R_{BWC; a’}(\beta) \geq R_{BWC; a}(\beta)
\end{equation*}
\end{lemma}

\begin{proof}
Reducing filter size $K$ will increase energy in high freqs i.e. $\forall K’ < K: \kappa_{high}(\beta;K’) > \kappa_{high}(\beta;K)$. This means that the space-limiting constraints only grow more stringent as we reduce $K$, implying that the result of the optimization problem for the implicit regularizer will only increase in cost i.e. $\forall K’ < K: R_{BWC; K’}(\beta) \geq R_{BWC; K}(\beta)$ for any candidate linear predictor $\beta$. (Note that this does not refer to the learned features $\beta^*$ which actually depends on the training data as well). In summary, all else being equal, reducing the kernel size $K$ causes/induces a bias towards more concentration of energy in higher frequencies in $\beta$.
\end{proof}
\noindent
In summary reducing the kernel size causes/induces a bias towards more concentration of energy in higher frequencies in $\beta$.

\end{document}